\documentclass{article} %
\usepackage[T1]{fontenc}
\usepackage{iclr2027_conference,times}

\usepackage{amsmath,amsfonts,bm}

\def\eqref#1{equation~\ref{#1}}

\def\1{\bm{1}}

\DeclareMathAlphabet{\mathsfit}{\encodingdefault}{\sfdefault}{m}{sl}
\SetMathAlphabet{\mathsfit}{bold}{\encodingdefault}{\sfdefault}{bx}{n}

\usepackage{graphicx}
\usepackage{url}
\usepackage[normalem]{ulem}
\usepackage{natbib}

\usepackage{amsmath, amssymb, amsthm}
\usepackage{pifont}
\newtheorem{theorem}{Theorem}       

                         \newtheorem{proposition}{Proposition}

\usepackage{graphicx} 
\usepackage{algorithm}                            \usepackage{algpseudocode}

\usepackage{booktabs,longtable}
\usepackage{xcolor}

\usepackage{multirow, multicol}

\usepackage{xcolor,colortbl}

\usepackage{wrapfig}
\usepackage{iftex}
\ifPDFTeX\else
\usepackage{newunicodechar}
\newunicodechar{’}{\textquoteright}
\newunicodechar{“}{\textquotedblleft}
\newunicodechar{”}{\textquotedblright}
\newunicodechar{–}{\textendash}
\fi
\usepackage{hyperref}

\definecolor{dgreen}{RGB}{0,150,0}
\definecolor{gray}{RGB}{127,127,127}
\definecolor{orange}{RGB}{255,120,0}

\newcommand{\bB}{\mathbf{B}}
\newcommand{\bA}{\mathbf{A}}
\newcommand{\bI}{\mathbf{I}}

\newcommand{\bS}{\mathbf{S}}

\newcommand{\bP}{\mathbf{P}}

\newcommand{\be}{\mathbf{e}}

\newcommand{\bx}{\mathbf{x}}
\newcommand{\bp}{\mathbf{p}}

\newcommand{\bt}{\mathbf{t}}
\newcommand{\bc}{\mathbf{c}}

\newcommand{\dirpix}{\text{ADR-}\ensuremath{\tilde{\pi}}}
\newcommand{\dirdense}{\text{ADR-}\ensuremath{\pi_d}}

\newcommand{\parag}[1]{\vspace{-3mm}\paragraph{#1}}

\title{Readout Blindness: VLM Scores Miss the Spatial Direction Their Frozen Encoders Retain}

\author{%
\parbox[t]{\dimexpr\textwidth-2\tabcolsep\relax}{\centering
Guangyuan Li\textsuperscript{1, 2}, 
Tianming Du\textsuperscript{1, 2}\thanks{\raggedright Correspondence to: Tianming Du (du.tianming@aalto.fi) and Jiancheng Yang (jiancheng.yang@aalto.fi).},
Yan Jiang,\textsuperscript{1, 3},  
Bihan Wen, \textsuperscript{4},
Jiancheng Yang\textsuperscript{1, 2}\footnotemark[1]
\\
\textsuperscript{1}ELLIS Institute Finland, \textsuperscript{2}Aalto University
\\
\textsuperscript{3}University of Oulu, \textsuperscript{4}Nanyang Technological University
\\
\texttt{\{guangyuan.li, du.tianming, jiancheng.yang\}@aalto.fi}
\\
\texttt{yan.jiang@oulu.fi}
\\
\texttt{bihan.wen@ntu.edu.sg}
}%
}

\iclrfinalcopy %
\definecolor{linkpink}{RGB}{170,45,85}
\hypersetup{colorlinks=true,
  linkcolor=linkpink,
  citecolor=linkpink,
  urlcolor=linkpink,
  pdfborder={0 0 0},
  pdftitle={Readout Blindness: VLM Scores Miss the Spatial Direction Their Frozen Encoders Retain},
  pdfauthor={Guangyuan Li, Tianming Du, Yan Jiang, Bihan Wen, Jiancheng Yang},
  pdfsubject={Preprint},
  pdfkeywords={vision-language models, spatial relations, readout, alignment}}
\begin{document}
\raggedbottom

\maketitle
\lhead{}
\renewcommand{\headrulewidth}{0pt}

\begin{abstract}

CLIP-like vision-language models remain a cornerstone of multimodal systems, yet their scores stay near chance on directed spatial relations, such as whether one object is left of another. We call this failure \textit{readout blindness} and analyze, theoretically and empirically, why deployed scores miss the direction: when scoring rules treat the subject and object symmetrically, direction cancels regardless of encoder training. Guided by this analysis, we introduce \textit{Antisymmetric Displacement Readout} (ADR), which aligns caption words with image patches in the frozen features and scores each relation by the signed displacement between matched object centroids. Notably, ADR succeeds without additional training or learned parameters, thereby demonstrating that directional information remains in the frozen encoder. However, text and world priors can inflate accuracy, so we further introduce \textit{prior deflation}, which measures the benefit of the image-text pairing as the grounded gain over a null that pairs each item with an unrelated image. Extensive experiments across encoder families show that ADR substantially improves over deployed scores, which remain near chance on most direction-balanced sets even for fine-tuned encoders. Compared with more complex readouts, ADR outperforms the evaluated MLLM likelihood readouts and is competitive with their chat inference at a small fraction of the computation. These results support our claim that directional information can be recovered from frozen features by an appropriate readout. Our implementation and evaluation kit are available at \textit{\url{https://github.com/GarryLarry010131/ADR}}.

\end{abstract}

\section{Introduction}
\label{sec:intro}
Vision-language models such as CLIP~\citep{radford2021learning} and SigLIP~\citep{zhai2023sigmoid} match images and captions through a score computed from separately encoded representations. Their vision encoders are also used in multimodal large language models (MLLMs), which generate answers~\citep{liu2024improved,bai2025qwen3}. This paper concerns the score-based readout. Beyond its efficiency, scores over explicit candidates support ranking and thresholding, making them a natural interface for decision making. This long-standing design~\citep{lecun2006tutorial} is also reflected in recent systems such as Jev~\citep{typesafe2026jev}. For such decisions to be reliable, the score must capture spatial relations between objects, since the meaning of a caption can depend on whether one object is to the left of, above, or in front of another.
However, existing image-text matching scores remain surprisingly weak at these directional relations, often performing near chance even when the relevant objects are recognized~\citep{yuksekgonul2023when,kamath2023whatsup,liu2023visual}.

Existing works have mainly attributed directional failure to limitations in training or representation, including contrastive objectives~\citep{yuksekgonul2023when}, limited spatial supervision~\citep{kamath2023whatsup}, encoder design, and learning mechanisms~\citep{alam2026spatial,yamamoto2026left}. Accordingly, most approaches modify the encoder through fine-tuning or additional supervision~\citep{zhang2024contrasting,castro2024clove,hoehing2023s}. Improvement alone, however, does not explain the original failure: it may come from newly learned spatial cues or from text and world priors. Meanwhile, structural analyses of global embeddings~\citep{kang2025clip} and spatial pooling~\citep{bukkapatnam2026compositional} explain specific failure mechanisms, but not the broader directional failures observed across deployed scoring rules. Therefore, the question remains unresolved: \textit{does the original encoder already contain the directional information that its score fails to expose?}

To answer this question, we keep the encoder fixed and examine what can be recovered from the features. We find that these features retain directional information, even when the image-text score fails to distinguish direction. We call this gap between what the encoder retains and what the score reveals \textit{readout blindness}. We support this diagnosis theoretically and empirically, showing that directional failure can arise from the scoring rule even when the information is present in the encoder. This shifts the problem from learning a better representation to designing a better readout.

\begin{figure}[t]
\centering
\includegraphics[width=\linewidth]{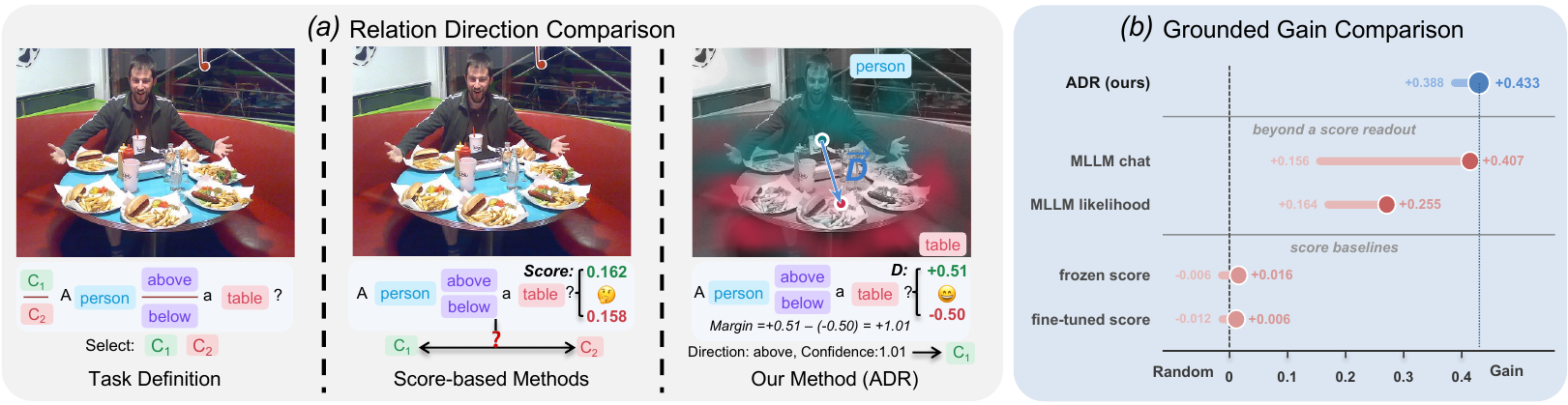}\par
\caption{Readout blindness and ADR. (a)~Choosing between captions with opposite relation directions. The deployed score gives near-equal scores to both, while ADR recovers the direction from the signed centroid displacement $D$ from the same frozen encoder. (b)~Grounded gain on VG-Relation left/right. Dots mark the best readout per group, and segments show the range.}

\label{fig:teaser}
\end{figure}

Therefore, we propose the \emph{Antisymmetric Displacement Readout} (ADR), which recovers relational direction directly from the same frozen encoder. ADR makes the correspondence between caption elements and visual regions explicit and determines direction from the signed displacement between object centroids (Fig.~\ref{fig:teaser}(a)). It requires neither tuning nor additional learned parameters, while its alignment maps and centroids help explain each decision.

In addition, we find that priors can drive high performance in readouts without the matched image, making accuracy alone unreliable for evaluating alignment. We therefore introduce prior deflation to assess the contribution of image-caption alignment. We pair each caption with an unrelated image from the same benchmark and score it again. We call this the derangement null and use it to assess prior effects. Prior deflation then subtracts null accuracy from raw accuracy to obtain grounded gain, which measures the benefit of the original image-caption pairing after accounting for prior effects. 
  
For fairer comparisons, we compare readouts on direction-balanced subsets with near-chance null accuracies. These subsets, prior deflation, reproducible derangement controls and shared scoring rules form our reusable evaluation kit. Under this evaluation, ADR achieves larger grounded gains than deployed readouts, including readouts from fine-tuned encoders. It is also competitive with the chat readouts of MLLMs (Fig.~\ref{fig:teaser}(b)). These results support our diagnosis of readout blindness: directional information can be recovered from frozen encoders by an appropriate readout.

In summary, our contributions are:
\textit{(i)} we identify and formalize \emph{readout blindness}, showing theoretically and empirically that directional failure can arise from the readout;
\textit{(ii)} we propose ADR, a training-free readout that recovers direction through explicit alignment and signed centroid displacement; and
\textit{(iii)} we introduce an evaluation kit combining prior deflation, direction-balanced subsets and shared scoring rules to compare readouts by grounded gain after accounting for prior effects.

\section{Related Work}
\label{sec:related}

\parag{Readouts of vision-language models.}
Vision language systems expose several readouts for image text matching. Dual encoders encode the image and text separately and compare the pooled representations with a global similarity~\citep{radford2021learning,zhai2023sigmoid}. Matching heads in BLIP~\citep{li2022blip}, BLIP-2~\citep{li2023blip} and FLAVA~\citep{singh2022flava} score image text pairs through cross-modal attention. MLLMs connect a vision encoder to a large language model and offer chat and likelihood readouts~\citep{liu2024improved,bai2025qwen3}, and VQAScore uses the probability of an affirmative answer to a matching question~\citep{lin2024evaluating}. These readouts differ in how the image and text interact, yet none explicitly compares object positions. We therefore treat the readout as the object of study and ask which scoring rules can express relation direction (Sec.~\ref{sec:impossibility}).

\parag{Spatial direction.}
Deployed scores remain weak on directional and compositional benchmarks \citep{yuksekgonul2023when,kamath2023whatsup,liu2023visual,thrush2022winoground,parcalabescu2022valse}. The remedies in Sec.~\ref{sec:intro} neither test what frozen features contain nor control for priors. Among structural analyses, \citet{kang2025clip} show that an idealized CLIP embedding cannot represent spatial relations and train DCSM on similarity maps. \citet{bukkapatnam2026compositional} attribute the failure to permutation-invariant pooling, but their accuracy bound assumes equal caption scores for one image, which their proof does not establish. Neither uses a prior control. We instead derive when a scoring rule cannot separate directions (Sec.~\ref{sec:impossibility}) and measure grounded gain (Sec.~\ref{sec:deflation}).

\parag{Explicit alignment.}
An explicit alignment is a correspondence between text tokens and image regions that a method computes and exposes. Optimal transport (OT) couplings play this role by matching patches or concepts across modalities~\citep{xie2025discovering,zhang2026bridging} and by supporting interpretation in language tasks~\citep{jia2026scout}. ReCLIP resolves spatial relations by comparing the centers of detector proposals~\citep{subramanian2022reclip}. These couplings carry individual concepts rather than directed relations and are read through trained modules. ReCLIP is training-free, but it takes direction from hand-coded box rules over external proposals rather than from the alignment itself. ADR builds on the same OT coupling (Sec.~\ref{sec:preliminaries}) and reads direction from the frozen encoder's own alignment through an antisymmetric kernel (Sec.~\ref{sec:method}).

\section{Preliminaries}
\label{sec:preliminaries}

A readout uses the image and text representations produced by a model to produce a task output. For image-text matching, it assigns a score that indicates how well an image and a caption match.

\parag{Features and coupling.}
For an image $I$, let $\bx_i$ ($i=1,\ldots,N$) be its unit-norm patch features. The text encoder maps each content word in caption $C$ independently to a feature $\bt_j$ ($j=1,\ldots,m$) in the same space. The patch-word similarity matrix $\bS$ has entries $S_{ij}=\langle\bx_i,\bt_j\rangle$. Entropic optimal transport~\citep{cuturi2013sinkhorn,gabriel2019computational} gives the coupling
\begin{equation}
\pi_0(\bS)=\arg\max_{\pi\in U(\mathbf{a},\mathbf{b})}
\big\{\langle\bS,\pi\rangle-\varepsilon\,\mathrm{KL}(\pi\,\|\,\mathbf{a}\otimes\mathbf{b})\big\},
\label{eq:ot}
\end{equation}
where $\varepsilon>0$, $\mathbf{a}=\mathbf{1}/N$, $\mathbf{b}=\mathbf{1}/m$, and $U(\mathbf{a},\mathbf{b})=\{\pi>0\mid\pi\mathbf{1}=\mathbf{a},\ \pi^\top\mathbf{1}=\mathbf{b}\}$. $\varepsilon$ controls the regularization strength, with smaller values giving more concentrated alignments.
Sinkhorn scaling~\citep{sinkhorn1967concerning} computes $\pi_0$ by alternately rescaling rows and columns of $\exp(\bS/\varepsilon)$.

\parag{Relations and symmetry.}
The image kernel $\bA\in\mathbb{R}^{N\times N}$ and text kernel $\bB\in\mathbb{R}^{m\times m}$ are matrices that encode relations between patches and between words, respectively. The coupling $\pi$ aligns image relations with word relations through $\pi^\top\bA\pi$, which we compare with $\bB$ using the score $\operatorname{tr}(\pi^\top\bA\pi\bB)$ (App.~\ref{app:relation-score}). To study whether this score distinguishes relation direction, we reverse word relations, which replaces $\bB$ with $\bB^\top$, and separate each kernel into symmetric and antisymmetric parts. We write $\bA=\bA^s+\bA^a$, where $\bA^s=(\bA+\bA^\top)/2$ and $\bA^a=(\bA-\bA^\top)/2$, and likewise for $\bB$. Transposition preserves the symmetric part and negates the antisymmetric part.

\section{Readout Blindness}
\label{sec:impossibility}

For an image $I$, a directional score $F(I,C)$ must rank the correct caption $C^+$ above a distractor $C^-$ with the subject and object swapped (an argument swap). These captions contain the same words but express different relations. We analyze the scoring rules behind these failures and prove that inappropriate readouts can cause direction blindness, which encoder training alone cannot resolve.

We first consider scores that ignore word order. We call a score $F(I,C)$ \emph{unary} if it treats the caption as a bag of words,depending only on which word features $\bt_j$ occur and how often, but not the order. These scores include dense pooling of $\bS$ and the transport score $\langle\bS,\pi_0\rangle$, with $\pi_0$ computed by Eq.~\ref{eq:ot}.

\begin{theorem}[Unary blindness]
\label{thm:unary}
If $C^+$ and $C^-$ have the same word multiset, every unary score satisfies $F(I,C^+)=F(I,C^-)$ for every image and every choice of context-free word features.
\end{theorem}

An argument swap does not change the multiset, so further encoder training cannot break the tie under a unary score (proof in App.~\ref{app:unary-proof}). This motivates adding word relations that encode the subject and object roles while preserving the argument order.

We next consider relational scores, which compare word relations $\bB$ with image relations through the coupling. However, with a symmetric image kernel and a fixed coupling, the score cannot distinguish opposite directions. Consider the combined score
\begin{equation}
F(I,C)=\alpha\,\langle\bS,\pi_0\rangle
 +(1-\alpha)\,\operatorname{tr}(\pi_0^\top\bA\pi_0\bB),
\label{eq:oneshot}
\end{equation}
where $\alpha\in[0,1]$ and $\pi_0$ is computed from $\bS$ alone by Eq.~\ref{eq:ot}. The first term is unary and covered by Thm.~\ref{thm:unary}. We therefore examine whether the second, relational term can break the tie. This term separates into symmetric and antisymmetric interactions (proof in App.~\ref{app:decomposition}):
\begin{equation}
\operatorname{tr}(\pi^\top\bA\pi\bB)
 =\operatorname{tr}(\pi^\top\bA^s\pi\bB^s)
 +\operatorname{tr}(\pi^\top\bA^a\pi\bB^a).
\label{eq:decomposition}
\end{equation}
Reversing a relation leaves $\bB^s$ unchanged and negates $\bB^a$. When the image kernel is symmetric ($\bA^a=0$), the antisymmetric term vanishes and the score cannot distinguish the two directions.

\begin{theorem}[Symmetric blindness]
\label{thm:relational}
For symmetric $\bA$ and fixed $\pi$, reversing the word relations leaves the score unchanged:
$\operatorname{tr}(\pi^\top\bA\pi\bB)=\operatorname{tr}(\pi^\top\bA\pi\bB^\top)$ for every $\bB$.
\end{theorem}

For symmetric $\bA$, Eq.~\ref{eq:oneshot} is therefore invariant to relation reversal at fixed $\bS$, for any $\alpha\in[0,1]$ and $\varepsilon>0$. Patch distances and inner-product similarities give symmetric image kernels, so further encoder training cannot remove this invariance (proof in App.~\ref{app:symmetry-proof}).

Thm.~\ref{thm:relational} compares a relation and its reversal using the same coupling. However, swapping the nouns in a caption also changes $\bS$ and therefore $\pi_0$. To study this change, we consider \emph{parse-isomorphic} argument swaps, which preserve the other words and the dependency structure. The swap permutes the columns of $\bS$ and the coupling $\pi_0$ from Eq.~\ref{eq:ot}. Assume the rows and columns of $\bB^s$ follow the same noun permutation. These permutations leave both terms of Eq.~\ref{eq:oneshot} unchanged when $\bA$ is symmetric, giving the following result.

\begin{theorem}[Chance by construction]
\label{thm:swap}
Under a parse-isomorphic argument swap, any unary score and any relational score with a symmetric image kernel (Eq.~\ref{eq:oneshot}) assign equal values to $C^+$ and $C^-$.
\end{theorem}

For these caption pairs, the tie follows from the scoring structure and cannot be broken by improving the features (proof in App.~\ref{app:swap-proof}).

Beyond the theorems, global cosine scores compare the normalized image feature $\bx(I)$ with the full-caption feature $\bt(C)$. Their margin is $F(I,C^+)-F(I,C^-)=\bx(I)^\top[\bt(C^+)-\bt(C^-)]$. This margin depends on how the image feature aligns with the caption difference, so different caption features do not guarantee a correct ranking. In our analysis, argument swaps change the caption features of both encoders, yet the scores produce no ties and accuracy remains near chance (App.~\ref{app:cosine-diagnostic}). These failures indicate that a score can respond to word order without correctly reading relation direction.

To better distinguish relation direction, alternative readouts allow the image and text to interact before scoring. Cross-encoder matching heads jointly process the image and caption, while MLLM likelihoods score answers using both inputs. These approaches partly recover directional signal, with gains varying across tasks (Tab.~\ref{tab:alternatives}). However, their scores still measure overall image-text compatibility or answer likelihood without explicitly comparing object positions, so they do not guarantee that the spatial relation is correct.

The theorems identify a structural form of readout blindness, where the scoring rule prevents directional discrimination regardless of feature values. Thus, encoder training alone cannot break these ties. The experiments also reveal directional failures in readouts beyond these theorems. Together, these findings motivate a readout that preserves relation direction through an antisymmetric image kernel and uses explicit alignment to compare object positions.

\section{Antisymmetric Displacement Readout}
\label{sec:method}
\label{sec:channel}

\begin{figure}[t]
\centering
\includegraphics[width=\textwidth]{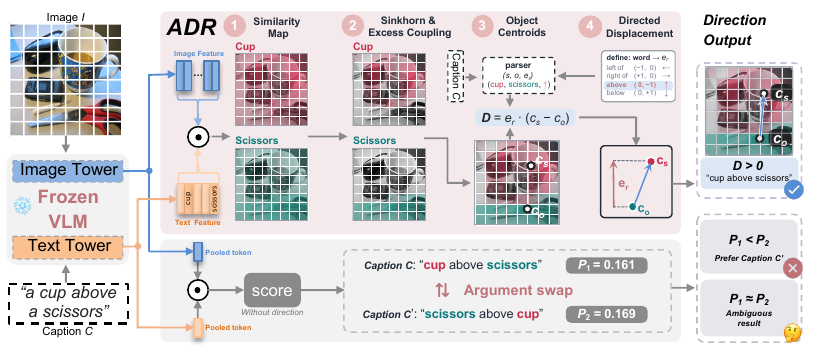}\par
\caption{ADR (top) versus deployed score (bottom) on the same frozen VLM. ADR (1)~computes patch--word similarity $\bS$, (2)~keeps the excess Sinkhorn coupling $\tilde{\pi}$, (3)~localizes the subject and object centroids, and (4)~scores their signed displacement $D=\be_r^\top(\bc_s-\bc_o)$ along the relation axis. The deployed score compares pooled features (CLS token for CLIP, attention-pooled token for SigLIP) and gives near-equal scores to a caption and its argument swap, leaving direction ambiguous.}

\label{fig:pipeline}
\end{figure}

Sec.~\ref{sec:impossibility} shows that, at a fixed coupling, the relation interaction requires a nonzero antisymmetric part in the image kernel to distinguish direction. We thus propose the antisymmetric displacement readout (ADR), which reads direction from the signed displacement between object centroids (Fig.~\ref{fig:pipeline}).

\parag{Antisymmetric kernels.}
To use the antisymmetric interaction in Eq.~\ref{eq:decomposition}, we construct antisymmetric image and text kernels. Let $\bp_i\in[0,1]^2$ be the normalized center of patch $i$ and $\be\in\mathbb{R}^2$ a unit vector along the relation axis. The image kernel has entries $A^a_{ii'}=\be^\top(\bp_i-\bp_{i'})$. On the text side, a directed relation from subject $s$ to object $o$ is represented by $B^a_{os}=1$ and $B^a_{so}=-1$, with all other entries zero. Both kernels change sign under argument exchange. This construction reduces the relational term to a weighted displacement between object centroids. For a nonnegative alignment $\pi$, let $\mu_j=\sum_i\pi_{ij}$ be the mass assigned to word $j$. Its centroid is 
$\bc_j=(\sum_i\pi_{ij}\bp_i)/\mu_j$, when $\mu_j>0$. With these kernels, the relational term reduces to a centroid displacement (proof in App.~\ref{app:centroid-proof}).
\begin{proposition}[Centroid displacement]
\label{prop:centroid}
For nonnegative $\pi$ with positive subject and object masses,
\[
\operatorname{tr}(\pi^\top\bA^a\pi\bB^a)
 =2\mu_s\mu_o\,\be^\top(\bc_s-\bc_o).
\]
\end{proposition}

The factor $2\mu_s\mu_o$ is positive, so each relation can be scored by the signed displacement $\be^\top(\bc_s-\bc_o)$. ADR computes this displacement directly without constructing the image and text kernels. We next specify how to compute the object centroids and choose the axis for each caption relation.

\parag{Object localization.}
ADR computes $\pi_0$ from $\bS$ alone using Eq.~\ref{eq:ot}, and each patch distributes a fixed total mass across the words. ADR then retains the mass above the uniform baseline as
$\tilde{\pi}=[\pi_0-1/(Nm)]_+$, where $[\cdot]_+$ clips negative entries to zero. Centroids computed from $\tilde{\pi}$ define \dirpix{}. App.~\ref{app:excess} describes how zero-mass columns are handled. Prop.~\ref{prop:centroid} also applies to other nonnegative weights with positive subject and object masses. A dense variant, \dirdense{}, computes centroids directly from weights $\pi_d$, obtained by a separate softmax of $\bS/\varepsilon$ over patches for each word. Each word is thus localized independently (App.~\ref{app:localizers}). Both variants use the same displacement score. We use \dirpix{} as the default because it is less sensitive to temperature (App.~\ref{app:epssweep}) and gives better confidence ranking (Tab.~\ref{tab:riskcov}).

\parag{Directional scoring.}

On the text side, a parsing rule maps each caption relation to a subject $s_r$, object $o_r$, and unit axis $\be_r$ as a triple $(s_r,o_r, \be_r)$, where $\be_r$ specifies the positive direction for each relation. For an image, ADR sums these displacements as $D(C)=\sum_r\be_r^\top(\bc_{s_r}-\bc_{o_r})$ and ranks candidate captions by $D$. By Prop.~\ref{prop:centroid}, the sign of each displacement indicates directional agreement, with zero as the threshold between opposite directions (App.~\ref{app:zero}). ADR can record the alignment weights, centroids and score to show which objects were localized and how their positions determine the prediction. Algorithm~\ref{alg:channel} summarizes the implementation of \dirpix{}, which uses the same displacement score across the evaluated encoders and datasets without fine-tuning or additional learned parameters. Sinkhorn settings are given in App.~\ref{app:sinkhorn}, and Tab.~\ref{tab:cost} reports the full readout cost.

\begin{table}[t]
\caption{Prior effects on a prior-biased subset (VALSE actant swap) and a direction-balanced subset (VG-Relation left/right). Acc. is accuracy with the original images. n.acc. is accuracy under the derangement null (Sec.~\ref{sec:deflation}). Random is accuracy from uniformly random predictions. Dev. ($\text{n.acc.}-\text{random}$) measures deviation from random guessing. Gain is grounded gain ($\text{acc.}-\text{n.acc.}$) and measures the benefit of the original image-caption alignment. Accuracies are reported in percent (\%), and dev. and gain in percentage points. Full results on all eight subsets are given in Tab.~\ref{tab:prior-all}.}
\label{tab:prior}
\centering
\fontsize{8}{9.5}\selectfont
\setlength{\tabcolsep}{3pt}
\begin{tabular*}{\columnwidth}{@{\extracolsep{\fill}}lclccccc@{}}
\toprule
\multirow{2}{*}{Subset} & \multirow{2}{*}{Random} & \multirow{2}{*}{Metric} & \multirow{2}{*}{CLIP} & \multirow{2}{*}{NegCLIP} & \multirow{2}{*}{SigLIP2} & LLaVA-1.5 & Qwen3.6 \\
 & & & & & & likelihood & likelihood \\
\midrule
\emph{Prior-based subset} &
 \multirow{4}{*}{50.03} & acc. & 72.71 & 85.14 & 62.07 & 93.26 & 93.99 \\
 \multirow{3}{*}{\shortstack[l]{VALSE\\actant swap}} & & n.acc. & 41.31 & 62.80 & 28.35 & 88.51 & 88.41 \\
& & dev. & $-8.7$ & $+12.8$ & $-21.7$ & $+38.5$ & $+38.4$ \\
& & gain & $+31.4$ & $+22.3$ & $+33.7$ & $+4.8$ & $+5.6$ \\
\midrule
\emph{Direction-balanced subset} & \multirow{4}{*}{50.01} & acc. & 49.64 & 50.26 & 49.95 & 65.80 & 74.76 \\
\multirow{3}{*}{\shortstack[l]{VG-Rel\\left/right}} & & n.acc. & 49.69 & 49.67 & 50.54 & 49.44 & 49.29 \\
& & dev. & $-0.3$ & $-0.3$ & $+0.5$ & $-0.6$ & $-0.7$ \\
& & gain & $-0.1$ & $+0.6$ & $-0.6$ & $+16.4$ & $+25.5$ \\
\bottomrule
\end{tabular*}
\end{table}

\section{Evaluation Kit}
\label{sec:deflation}
\label{sec:deflation-results}
\begin{wrapfigure}[23]{R}{0.3\textwidth}
\centering
\includegraphics[width=\linewidth]{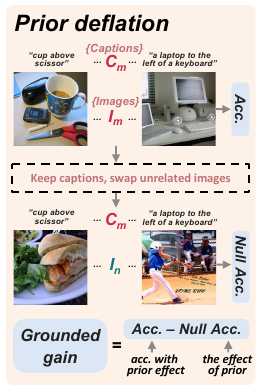}
\par
\caption{\raggedright Prior deflation compares accuracy under the original pairing and a derangement null. Their difference is grounded gain.}
\label{fig:prior-deflation}
\end{wrapfigure}

ADR in Sec.~\ref{sec:method} recovers direction without fine-tuning or additional learned parameters. We next assess how much of the performance comes from image-text alignment, because a readout can use text and world priors to favor the correct answer even without the paired image. High accuracy alone therefore does not show that a readout correctly aligns the image and caption. We introduce an evaluation kit that combines prior deflation, direction-balanced evaluation sets and shared scoring rules. Prior deflation uses a derangement null to assess prior effects and reports grounded gain as a measure of image-text alignment after accounting for these effects (Fig.~\ref{fig:prior-deflation}).

\parag{Derangement null.}
The derangement null replaces each image with another from the same benchmark through a deterministic permutation. No item receives its original source image. All readouts use the same replacements, while captions, labels, model weights and decision procedures remain unchanged. Null accuracy measures performance with the deranged pairing. If null accuracy deviates from chance, the readout prefers certain answers even without the original image, which may reflect text or world priors.

\parag{Prior effects.}
Tab.~\ref{tab:prior} compares five readouts on VALSE actant swap and VG-Relation left/right. The deviation is derangement-null accuracy minus random-guess accuracy. This measures the degree of prior bias under this control, while its sign shows whether that bias supports or opposes the correct choice. On VALSE actant swap, all five readouts show large deviations, indicating substantial prior effects. On VG-Relation left/right, all five null accuracies remain close to random guessing, allowing comparisons with less influence from priors. The other subsets show a similar contrast between prior-biased and direction-balanced evaluation sets (Tab.~\ref{tab:prior-all}).

\parag{Grounded gain.}
These results show why raw accuracy alone cannot measure the benefit of the original pairing. Prior deflation therefore reports the grounded gain as $\mathrm{gain}=\mathrm{Acc}_{\mathrm{real}}-\mathrm{Acc}_{\mathrm{null}}$, where the two accuracies use the original pairing and the derangement null, respectively. A positive gain indicates a benefit from the paired image. Each readout is compared with its own null because readouts exploit prior preferences to different degrees. We report both accuracies alongside gain. Grounded gain on the prior-biased subset in Tab.~\ref{tab:prior} can be small even when raw accuracy is high. On VALSE actant swap, LLaVA-1.5 likelihood achieves $93.26\%$ real accuracy but only $+4.8$ percentage points of gain because its null accuracy reaches $88.51\%$. Accuracy remains high without the original image, limiting what this result reveals about image-text alignment. On VG-Relation left/right, the same readout achieves $65.80\%$ real accuracy and $49.44\%$ null accuracy, giving a grounded gain of $+16.4$ points despite lower raw accuracy. Reducing prior bias makes it easier to assess the contribution of image-text alignment and identify differences between readouts. We therefore construct direction-balanced evaluation sets from existing benchmarks.

\parag{Evaluation protocol.} 

We define candidate subsets by relation labels and task structure, and evaluate each under the derangement null with the MLLM likelihood baselines (LLaVA and Qwen), which show strong prior effects in Tab.~\ref{tab:prior}. Subsets with near-chance null accuracies for both baselines form our direction-balanced evaluation sets (Tab.~\ref{tab:data-overview}; nulls in Tab.~\ref{tab:prior-all}). All readouts use the same items, derangement assignments and scoring rules (App.~\ref{app:scoring}). We will release the item lists and evaluation tools as a reusable kit (App.~\ref{app:data}).

\providecommand{\pizero}{\pi_0}
\providecommand{\dirpix}{\text{ADR-}\ensuremath{\tilde{\pi}}}
\providecommand{\dirpi}{\text{ADR-}\ensuremath{\pi_0}}
\providecommand{\dirdense}{\text{ADR-}\ensuremath{\pi_d}}
\providecommand{\dirargmax}{\text{ADR-argmax}}
\providecommand{\tr}{\operatorname{tr}}
\providecommand{\ci}[2]{{\footnotesize[#1,\,#2]}}

\section{Experiments}
\label{sec:experiments}

\label{sec:setup}

\begin{table}[t]
\caption{Random guess, deployed readouts and ADR across nine model families on five direction-balanced sets. Each subset reports accuracy with the original images (acc.), derangement-null accuracy (n.acc.), and grounded gain (gain). Accuracies are in percent (\%) and gains in percentage points. \textcolor[gray]{0.45}{Grey} rows show deployed readouts, followed by \dirpix{} and \dirdense{} on the same encoder and items. \colorbox{red!25}{Red} highlights grounded gain, the main metric for comparing readouts.}
\label{tab:main}
\centering
\fontsize{8}{9.5}\selectfont
\setlength{\tabcolsep}{1pt}
\begin{tabular*}{\columnwidth}{@{\extracolsep{\fill}}l*{15}{c}@{}}
\toprule
\multirow{2}{*}{Readout} & \multicolumn{3}{c}{VG-Rel L/R} & \multicolumn{3}{c}{COCO} & \multicolumn{3}{c}{WU-A} & \multicolumn{3}{c}{WU-B} & \multicolumn{3}{c}{VSR} \\
\cmidrule(lr){2-4}\cmidrule(lr){5-7}\cmidrule(lr){8-10}\cmidrule(lr){11-13}\cmidrule(lr){14-16}
 & acc. & n.acc. & gain & acc. & n.acc. & gain & acc. & n.acc. & gain & acc. & n.acc. & gain & acc. & n.acc. & gain \\
\midrule
Random & 50.01 & 50.01 & 0.0 & 49.99 & 49.99 & 0.0 & 25.02 & 25.02 & 0.0 & 50.01 & 50.01 & 0.0 & 50.01 & 50.01 & 0.0 \\
\midrule

\textcolor[gray]{0.45}{CLIP} & \textcolor[gray]{0.45}{49.64} & \textcolor[gray]{0.45}{49.69} & \cellcolor{red!25}\textcolor[gray]{0.45}{$-0.1$} & \textcolor[gray]{0.45}{45.30} & \textcolor[gray]{0.45}{47.28} & \cellcolor{red!25}\textcolor[gray]{0.45}{$-2.0$} & \textcolor[gray]{0.45}{26.96} & \textcolor[gray]{0.45}{25.98} & \cellcolor{red!25}\textcolor[gray]{0.45}{$+1.0$} & \textcolor[gray]{0.45}{53.67} & \textcolor[gray]{0.45}{52.83} & \cellcolor{red!25}\textcolor[gray]{0.45}{$+0.8$} & \textcolor[gray]{0.45}{53.87} & \textcolor[gray]{0.45}{53.21} & \cellcolor{red!25}\textcolor[gray]{0.45}{$+0.7$} \\
\quad \dirpix{} & 90.45 & 49.03 & \cellcolor{red!25}$+41.4$ & 92.33 & 51.73 & \cellcolor{red!25}$+40.6$ & 31.62 & 25.74 & \cellcolor{red!25}$+5.9$ & 93.22 & 52.26 & \cellcolor{red!25}$+41.0$ & 73.86 & 47.79 & \cellcolor{red!25}$+26.1$ \\
\quad \dirdense{} & 90.22 & 49.44 & \cellcolor{red!25}$+40.8$ & 94.06 & 52.72 & \cellcolor{red!25}$+41.3$ & 41.18 & 25.25 & \cellcolor{red!25}$+15.9$ & 98.02 & 50.00 & \cellcolor{red!25}$+48.0$ & 74.94 & 47.88 & \cellcolor{red!25}$+27.1$ \\
\textcolor[gray]{0.45}{LAION-2B} & \textcolor[gray]{0.45}{49.13} & \textcolor[gray]{0.45}{49.77} & \cellcolor{red!25}\textcolor[gray]{0.45}{$-0.6$} & \textcolor[gray]{0.45}{43.07} & \textcolor[gray]{0.45}{47.28} & \cellcolor{red!25}\textcolor[gray]{0.45}{$-4.2$} & \textcolor[gray]{0.45}{27.70} & \textcolor[gray]{0.45}{26.23} & \cellcolor{red!25}\textcolor[gray]{0.45}{$+1.5$} & \textcolor[gray]{0.45}{50.00} & \textcolor[gray]{0.45}{50.28} & \cellcolor{red!25}\textcolor[gray]{0.45}{$-0.3$} & \textcolor[gray]{0.45}{51.96} & \textcolor[gray]{0.45}{52.96} & \cellcolor{red!25}\textcolor[gray]{0.45}{$-1.0$} \\
\quad \dirpix{} & 89.05 & 50.23 & \cellcolor{red!25}$+38.8$ & 91.83 & 51.73 & \cellcolor{red!25}$+40.1$ & 31.86 & 23.28 & \cellcolor{red!25}$+8.6$ & 97.74 & 52.26 & \cellcolor{red!25}$+45.5$ & 74.19 & 48.96 & \cellcolor{red!25}$+25.2$ \\
\quad \dirdense{} & 88.41 & 50.13 & \cellcolor{red!25}$+38.3$ & 93.56 & 51.24 & \cellcolor{red!25}$+42.3$ & 38.73 & 25.25 & \cellcolor{red!25}$+13.5$ & 99.15 & 52.54 & \cellcolor{red!25}$+46.6$ & 74.94 & 50.04 & \cellcolor{red!25}$+24.9$ \\
\textcolor[gray]{0.45}{DFN5B} & \textcolor[gray]{0.45}{50.46} & \textcolor[gray]{0.45}{50.89} & \cellcolor{red!25}\textcolor[gray]{0.45}{$-0.4$} & \textcolor[gray]{0.45}{44.80} & \textcolor[gray]{0.45}{48.51} & \cellcolor{red!25}\textcolor[gray]{0.45}{$-3.7$} & \textcolor[gray]{0.45}{28.92} & \textcolor[gray]{0.45}{27.21} & \cellcolor{red!25}\textcolor[gray]{0.45}{$+1.7$} & \textcolor[gray]{0.45}{53.11} & \textcolor[gray]{0.45}{51.69} & \cellcolor{red!25}\textcolor[gray]{0.45}{$+1.4$} & \textcolor[gray]{0.45}{54.04} & \textcolor[gray]{0.45}{52.21} & \cellcolor{red!25}\textcolor[gray]{0.45}{$+1.8$} \\
\quad \dirpix{} & 91.22 & 47.93 & \cellcolor{red!25}$+43.3$ & 97.28 & 50.99 & \cellcolor{red!25}$+46.3$ & 26.96 & 25.49 & \cellcolor{red!25}$+1.5$ & 98.59 & 51.13 & \cellcolor{red!25}$+47.5$ & 75.85 & 48.46 & \cellcolor{red!25}$+27.4$ \\
\quad \dirdense{} & 91.88 & 48.95 & \cellcolor{red!25}$+42.9$ & 98.27 & 51.24 & \cellcolor{red!25}$+47.0$ & 63.24 & 24.75 & \cellcolor{red!25}$+38.5$ & 100.00 & 51.69 & \cellcolor{red!25}$+48.3$ & 77.10 & 46.96 & \cellcolor{red!25}$+30.1$ \\
\textcolor[gray]{0.45}{MetaCLIP} & \textcolor[gray]{0.45}{50.33} & \textcolor[gray]{0.45}{48.72} & \cellcolor{red!25}\textcolor[gray]{0.45}{$+1.6$} & \textcolor[gray]{0.45}{48.51} & \textcolor[gray]{0.45}{49.75} & \cellcolor{red!25}\textcolor[gray]{0.45}{$-1.2$} & \textcolor[gray]{0.45}{27.70} & \textcolor[gray]{0.45}{23.28} & \cellcolor{red!25}\textcolor[gray]{0.45}{$+4.4$} & \textcolor[gray]{0.45}{51.41} & \textcolor[gray]{0.45}{49.44} & \cellcolor{red!25}\textcolor[gray]{0.45}{$+2.0$} & \textcolor[gray]{0.45}{54.62} & \textcolor[gray]{0.45}{53.96} & \cellcolor{red!25}\textcolor[gray]{0.45}{$+0.7$} \\
\quad \dirpix{} & 90.76 & 49.34 & \cellcolor{red!25}$+41.4$ & 92.57 & 48.51 & \cellcolor{red!25}$+44.1$ & 26.72 & 25.25 & \cellcolor{red!25}$+1.5$ & 93.50 & 46.33 & \cellcolor{red!25}$+47.2$ & 73.86 & 48.88 & \cellcolor{red!25}$+25.0$ \\
\quad \dirdense{} & 91.04 & 49.62 & \cellcolor{red!25}$+41.4$ & 94.80 & 48.02 & \cellcolor{red!25}$+46.8$ & 52.70 & 25.00 & \cellcolor{red!25}$+27.7$ & 100.00 & 51.42 & \cellcolor{red!25}$+48.6$ & 74.85 & 49.71 & \cellcolor{red!25}$+25.1$ \\
\textcolor[gray]{0.45}{SigLIP2} & \textcolor[gray]{0.45}{49.95} & \textcolor[gray]{0.45}{50.54} & \cellcolor{red!25}\textcolor[gray]{0.45}{$-0.6$} & \textcolor[gray]{0.45}{49.01} & \textcolor[gray]{0.45}{45.05} & \cellcolor{red!25}\textcolor[gray]{0.45}{$+4.0$} & \textcolor[gray]{0.45}{29.17} & \textcolor[gray]{0.45}{24.75} & \cellcolor{red!25}\textcolor[gray]{0.45}{$+4.4$} & \textcolor[gray]{0.45}{51.69} & \textcolor[gray]{0.45}{48.87} & \cellcolor{red!25}\textcolor[gray]{0.45}{$+2.8$} & \textcolor[gray]{0.45}{54.70} & \textcolor[gray]{0.45}{52.96} & \cellcolor{red!25}\textcolor[gray]{0.45}{$+1.7$} \\
\quad \dirpix{} & 92.09 & 49.46 & \cellcolor{red!25}$+42.6$ & 96.53 & 52.48 & \cellcolor{red!25}$+44.1$ & 74.26 & 26.96 & \cellcolor{red!25}$+47.3$ & 100.00 & 50.85 & \cellcolor{red!25}$+49.2$ & 76.10 & 48.96 & \cellcolor{red!25}$+27.1$ \\
\quad \dirdense{} & 92.32 & 50.36 & \cellcolor{red!25}$+42.0$ & 97.52 & 51.24 & \cellcolor{red!25}$+46.3$ & 80.39 & 24.51 & \cellcolor{red!25}$+55.9$ & 100.00 & 51.42 & \cellcolor{red!25}$+48.6$ & 77.35 & 49.63 & \cellcolor{red!25}$+27.7$ \\
\textcolor[gray]{0.45}{BLIP} & \textcolor[gray]{0.45}{49.31} & \textcolor[gray]{0.45}{49.00} & \cellcolor{red!25}\textcolor[gray]{0.45}{$+0.3$} & \textcolor[gray]{0.45}{45.79} & \textcolor[gray]{0.45}{47.52} & \cellcolor{red!25}\textcolor[gray]{0.45}{$-1.7$} & \textcolor[gray]{0.45}{27.70} & \textcolor[gray]{0.45}{25.74} & \cellcolor{red!25}\textcolor[gray]{0.45}{$+2.0$} & \textcolor[gray]{0.45}{52.82} & \textcolor[gray]{0.45}{49.43} & \cellcolor{red!25}\textcolor[gray]{0.45}{$+3.4$} & \textcolor[gray]{0.45}{52.96} & \textcolor[gray]{0.45}{52.21} & \cellcolor{red!25}\textcolor[gray]{0.45}{$+0.8$} \\
\quad \dirpix{} & 89.51 & 48.47 & \cellcolor{red!25}$+41.0$ & 90.59 & 52.72 & \cellcolor{red!25}$+37.9$ & 27.70 & 27.94 & \cellcolor{red!25}$-0.2$ & 96.61 & 51.13 & \cellcolor{red!25}$+45.5$ & 72.94 & 49.63 & \cellcolor{red!25}$+23.3$ \\
\quad \dirdense{} & 88.97 & 49.11 & \cellcolor{red!25}$+39.9$ & 95.79 & 51.24 & \cellcolor{red!25}$+44.6$ & 41.91 & 25.00 & \cellcolor{red!25}$+16.9$ & 100.00 & 50.56 & \cellcolor{red!25}$+49.4$ & 75.69 & 48.79 & \cellcolor{red!25}$+26.9$ \\
\textcolor[gray]{0.45}{NegCLIP} & \textcolor[gray]{0.45}{50.26} & \textcolor[gray]{0.45}{49.67} & \cellcolor{red!25}\textcolor[gray]{0.45}{$+0.6$} & \textcolor[gray]{0.45}{44.80} & \textcolor[gray]{0.45}{49.50} & \cellcolor{red!25}\textcolor[gray]{0.45}{$-4.7$} & \textcolor[gray]{0.45}{33.33} & \textcolor[gray]{0.45}{26.23} & \cellcolor{red!25}\textcolor[gray]{0.45}{$+7.1$} & \textcolor[gray]{0.45}{48.31} & \textcolor[gray]{0.45}{51.69} & \cellcolor{red!25}\textcolor[gray]{0.45}{$-3.4$} & \textcolor[gray]{0.45}{52.87} & \textcolor[gray]{0.45}{52.79} & \cellcolor{red!25}\textcolor[gray]{0.45}{$+0.1$} \\
\quad \dirpix{} & 89.13 & 48.65 & \cellcolor{red!25}$+40.5$ & 91.83 & 53.71 & \cellcolor{red!25}$+38.1$ & 28.68 & 24.26 & \cellcolor{red!25}$+4.4$ & 93.79 & 47.18 & \cellcolor{red!25}$+46.6$ & 73.94 & 49.54 & \cellcolor{red!25}$+24.4$ \\
\quad \dirdense{} & 89.13 & 49.18 & \cellcolor{red!25}$+40.0$ & 93.07 & 52.97 & \cellcolor{red!25}$+40.1$ & 33.82 & 25.00 & \cellcolor{red!25}$+8.8$ & 96.89 & 49.15 & \cellcolor{red!25}$+47.7$ & 74.10 & 49.46 & \cellcolor{red!25}$+24.6$ \\
\textcolor[gray]{0.45}{CLoVe} & \textcolor[gray]{0.45}{48.83} & \textcolor[gray]{0.45}{50.05} & \cellcolor{red!25}\textcolor[gray]{0.45}{$-1.2$} & \textcolor[gray]{0.45}{51.49} & \textcolor[gray]{0.45}{50.99} & \cellcolor{red!25}\textcolor[gray]{0.45}{$+0.5$} & \textcolor[gray]{0.45}{34.56} & \textcolor[gray]{0.45}{23.77} & \cellcolor{red!25}\textcolor[gray]{0.45}{$+10.8$} & \textcolor[gray]{0.45}{48.59} & \textcolor[gray]{0.45}{47.74} & \cellcolor{red!25}\textcolor[gray]{0.45}{$+0.9$} & \textcolor[gray]{0.45}{50.96} & \textcolor[gray]{0.45}{51.79} & \cellcolor{red!25}\textcolor[gray]{0.45}{$-0.8$} \\
\quad \dirpix{} & 88.57 & 48.29 & \cellcolor{red!25}$+40.3$ & 90.84 & 55.45 & \cellcolor{red!25}$+35.4$ & 30.64 & 25.25 & \cellcolor{red!25}$+5.4$ & 95.20 & 51.41 & \cellcolor{red!25}$+43.8$ & 74.69 & 48.96 & \cellcolor{red!25}$+25.7$ \\
\quad \dirdense{} & 88.51 & 48.77 & \cellcolor{red!25}$+39.7$ & 93.32 & 52.72 & \cellcolor{red!25}$+40.6$ & 38.24 & 26.23 & \cellcolor{red!25}$+12.0$ & 99.43 & 51.13 & \cellcolor{red!25}$+48.3$ & 75.52 & 49.29 & \cellcolor{red!25}$+26.2$ \\
\textcolor[gray]{0.45}{CE-CLIP} & \textcolor[gray]{0.45}{51.30} & \textcolor[gray]{0.45}{50.94} & \cellcolor{red!25}\textcolor[gray]{0.45}{$+0.4$} & \textcolor[gray]{0.45}{52.48} & \textcolor[gray]{0.45}{50.74} & \cellcolor{red!25}\textcolor[gray]{0.45}{$+1.7$} & \textcolor[gray]{0.45}{31.13} & \textcolor[gray]{0.45}{27.45} & \cellcolor{red!25}\textcolor[gray]{0.45}{$+3.7$} & \textcolor[gray]{0.45}{50.00} & \textcolor[gray]{0.45}{49.44} & \cellcolor{red!25}\textcolor[gray]{0.45}{$+0.6$} & \textcolor[gray]{0.45}{53.29} & \textcolor[gray]{0.45}{52.71} & \cellcolor{red!25}\textcolor[gray]{0.45}{$+0.6$} \\
\quad \dirpix{} & 90.20 & 48.21 & \cellcolor{red!25}$+42.0$ & 92.82 & 55.69 & \cellcolor{red!25}$+37.1$ & 30.15 & 25.00 & \cellcolor{red!25}$+5.2$ & 92.09 & 48.87 & \cellcolor{red!25}$+43.2$ & 74.85 & 49.46 & \cellcolor{red!25}$+25.4$ \\
\quad \dirdense{} & 89.84 & 48.34 & \cellcolor{red!25}$+41.5$ & 91.83 & 54.70 & \cellcolor{red!25}$+37.1$ & 31.13 & 25.25 & \cellcolor{red!25}$+5.9$ & 92.94 & 49.72 & \cellcolor{red!25}$+43.2$ & 74.77 & 49.21 & \cellcolor{red!25}$+25.6$ \\
\bottomrule
\end{tabular*}
\end{table}

\parag{Data.} We evaluate on eight subsets from ARO~\citep{yuksekgonul2023when}, What'sUp~\citep{kamath2023whatsup}, VSR~\citep{liu2023visual} and VALSE~\citep{parcalabescu2022valse}. The main and alternative readout comparisons use five direction-balanced subsets: VG-Relation left/right, COCO-spatial, What'sUp-A, What'sUp-B and VSR horizontal. What'sUp-A has four answer choices; the other four subsets are binary. The remaining three subsets are used to examine prior effects in Sec.~\ref{sec:deflation}. Full dataset details are given in Tab.~\ref{tab:data-overview}.

\parag{Implementation.} We compare ADR with deployed readouts across nine model families (Tab.~\ref{tab:main}), including CLIP~\citep{radford2021learning}, SigLIP2~\citep{tschannen2025siglip} and compositional fine-tunes~\citep{yuksekgonul2023when,castro2024clove,zhang2024contrasting}. Alternative baselines include image-text matching heads~\citep{li2022blip,li2023blip,singh2022flava}, the half-crop readout (App.~\ref{app:halfcrop}), and likelihood and chat readouts of LLaVA-1.5~\citep{liu2024improved} and Qwen3.6~\mbox{\citep{qwen36_27b}} (Tab.~\ref{tab:alternatives}). The full model list and evaluation coverage are given in App.~\ref{app:models}.
All model weights remain frozen during evaluation. We use both \dirpix{} and \dirdense{} with $\varepsilon=0.05$ fixed across encoders and datasets. Sinkhorn scaling uses at most $200$ iterations with a stopping tolerance of $10^{-9}$. Feature extraction follows App.~\ref{app:poolingrule}, and MLLM settings are given in App.~\ref{app:mllmforms}.

\parag{Metrics.} Following the evaluation protocol in Sec.~\ref{sec:deflation}, we use grounded gain as the main measure for comparing readouts and report accuracy under the original pairing and the derangement null. For VSR baselines using oracle thresholds, real and null accuracies are optimized separately (App.~\ref{app:scoring}).

\subsection{Main Results}
\label{sec:grid}

\begin{table}[t]
\caption{Deployed readouts, alternative readouts and ADR on five direction-balanced sets. Each subset reports accuracy with the original images (acc.), derangement-null accuracy (n.acc.), and grounded gain (gain). Accuracies in percent (\%) and gains in percentage points. All readouts use the same evaluation items as Tab.~\ref{tab:main}. \colorbox{red!25}{Red} highlights grounded gain. \textbf{Bold} marks the highest gain.}
\label{tab:alternatives}
\centering
\fontsize{8}{9.5}
\selectfont
\setlength{\tabcolsep}{0.2pt}
\begin{tabular*}{\columnwidth}{@{\extracolsep{\fill}}l*{15}{c}@{}}
\toprule
\multirow{2}{*}{Readout} & \multicolumn{3}{c}{VG-Rel L/R} & \multicolumn{3}{c}{COCO} & \multicolumn{3}{c}{WU-A} & \multicolumn{3}{c}{WU-B} & \multicolumn{3}{c}{VSR} \\
\cmidrule(lr){2-4}\cmidrule(lr){5-7}\cmidrule(lr){8-10}\cmidrule(lr){11-13}\cmidrule(lr){14-16}
 & acc. & n.acc. & gain & acc. & n.acc. & gain & acc. & n.acc. & gain & acc. & n.acc. & gain & acc. & n.acc. & gain \\
\midrule
\multicolumn{16}{@{}l}{\emph{Deployed readouts}} \\
CLIP &
49.64 &
49.69 &
\cellcolor{red!25}$-0.1$ &
45.30 &
47.28 &
\cellcolor{red!25}$-2.0$ &
26.96 &
25.98 &
\cellcolor{red!25}$+1.0$ &
53.67 &
52.83 &
\cellcolor{red!25}$+0.8$ &
53.87 &
53.21 &
\cellcolor{red!25}$+0.7$ \\
SigLIP2 &
49.95 &
50.54 &
\cellcolor{red!25}$-0.6$ &
49.01 &
45.05 &
\cellcolor{red!25}$+4.0$ &
29.17 &
24.75 &
\cellcolor{red!25}$+4.4$ &
51.69 &
48.87 &
\cellcolor{red!25}$+2.8$ &
54.70 &
52.96 &
\cellcolor{red!25}$+1.7$ \\
\midrule
\multicolumn{16}{@{}l}{\emph{Half-crop}} \\
CLIP & 83.67 & 49.16 & \cellcolor{red!25}$+34.5$ & 82.92 & 55.45 & \cellcolor{red!25}$+27.5$ & 44.61 & 22.30 & \cellcolor{red!25}$+22.3$ & 100.00 & 50.00 & \cellcolor{red!25}$\mathbf{+50.0}$ & 70.52 & 47.88 & \cellcolor{red!25}$+22.6$ \\
SigLIP2 & 88.46 & 49.69 & \cellcolor{red!25}$+38.8$ & 91.83 & 49.50 & \cellcolor{red!25}$+42.3$ & 57.35 & 25.25 & \cellcolor{red!25}$+32.1$ & 99.72 & 51.70 & \cellcolor{red!25}$+48.0$ & 75.94 & 50.12 & \cellcolor{red!25}$+25.8$ \\
\midrule
\multicolumn{16}{@{}l}{\emph{Cross-encoder matching heads}} \\
BLIP-ITM base & 51.02 & 50.79 & \cellcolor{red!25}$+0.2$ & 55.94 & 48.27 & \cellcolor{red!25}$+7.7$ & 35.78 & 26.23 & \cellcolor{red!25}$+9.6$ & 55.37 & 49.15 & \cellcolor{red!25}$+6.2$ & 54.29 & 50.96 & \cellcolor{red!25}$+3.3$ \\
BLIP-ITM large & 50.38 & 49.62 & \cellcolor{red!25}$+0.8$ & 55.20 & 47.77 & \cellcolor{red!25}$+7.4$ & 31.74 & 24.51 & \cellcolor{red!25}$+7.2$ & 60.45 & 46.89 & \cellcolor{red!25}$+13.6$ & 54.29 & 50.96 & \cellcolor{red!25}$+3.3$ \\
BLIP-2 ViT-g & 49.90 & 49.36 & \cellcolor{red!25}$+0.5$ & 58.66 & 50.00 & \cellcolor{red!25}$+8.7$ & 44.61 & 21.32 & \cellcolor{red!25}$+23.3$ & 59.32 & 48.87 & \cellcolor{red!25}$+10.5$ & 54.04 & 51.37 & \cellcolor{red!25}$+2.7$ \\
FLAVA-full & 50.28 & 50.87 & \cellcolor{red!25}$-0.6$ & 54.70 & 51.49 & \cellcolor{red!25}$+3.2$ & 30.15 & 26.47 & \cellcolor{red!25}$+3.7$ & 51.13 & 48.02 & \cellcolor{red!25}$+3.1$ & 52.29 & 50.87 & \cellcolor{red!25}$+1.4$ \\
\midrule
\multicolumn{16}{@{}l}{\emph{MLLM likelihood}} \\
LLaVA-1.5-7B & 65.80 & 49.44 & \cellcolor{red!25}$+16.4$ & 79.21 & 54.21 & \cellcolor{red!25}$+25.0$ & 30.39 & 24.26 & \cellcolor{red!25}$+6.1$ & 63.56 & 52.83 & \cellcolor{red!25}$+10.7$ & 55.45 & 52.96 & \cellcolor{red!25}$+2.5$ \\
Qwen3.6-27B & 74.76 & 49.29 & \cellcolor{red!25}$+25.5$ & 88.12 & 49.01 & \cellcolor{red!25}$+39.1$ & 67.89 & 24.51 & \cellcolor{red!25}$+43.4$ & 79.09 & 50.56 & \cellcolor{red!25}$+28.5$ & 61.62 & 52.96 & \cellcolor{red!25}$+8.7$ \\
\midrule
\multicolumn{16}{@{}l}{\emph{MLLM chat}} \\
LLaVA-1.5-7B & 65.74 & 50.10 & \cellcolor{red!25}$+15.6$ & 65.59 & 50.74 & \cellcolor{red!25}$+14.9$ & 44.00 & 21.45 & \cellcolor{red!25}$+22.6$ & 65.11 & 51.41 & \cellcolor{red!25}$+13.7$ & 63.70 & 50.04 & \cellcolor{red!25}$+13.7$ \\
Qwen3.6-27B & 89.74 & 49.07 & \cellcolor{red!25}$+40.7$ & 96.16 & 51.11 & \cellcolor{red!25}$+45.1$ & 99.39 & 32.97 & \cellcolor{red!25}$\mathbf{+66.4}$ & 100.00 & 50.00 & \cellcolor{red!25}$\mathbf{+50.0}$ & 88.01 & 49.71 & \cellcolor{red!25}$\mathbf{+38.3}$ \\
\midrule
\multicolumn{16}{@{}l}{\emph{ADR on CLIP}} \\
\dirpix{} & 90.45 & 49.03 & \cellcolor{red!25}$+41.4$ & 92.33 & 51.73 & \cellcolor{red!25}$+40.6$ & 31.62 & 25.74 & \cellcolor{red!25}$+5.9$ & 93.22 & 52.26 & \cellcolor{red!25}$+41.0$ & 73.86 & 47.79 & \cellcolor{red!25}$+26.1$ \\
\dirdense{} &
90.22 &
49.44 &
\cellcolor{red!25}$+40.8$ &
94.06 &
52.72 &
\cellcolor{red!25}$+41.3$ &
41.18 &
25.25 &
\cellcolor{red!25}$+15.9$ &
98.02 &
50.00 &
\cellcolor{red!25}$+48.0$ &
74.94 &
47.88 &
\cellcolor{red!25}$+27.1$ \\
\midrule
\multicolumn{16}{@{}l}{\emph{ADR on SigLIP2}} \\
\dirpix{} & 92.09 & 49.46 & \cellcolor{red!25}$\mathbf{+42.6}$ & 96.53 & 52.48 & \cellcolor{red!25}$+44.1$ & 74.26 & 26.96 & \cellcolor{red!25}$+47.3$ & 100.00 & 50.85 & \cellcolor{red!25}$+49.2$ & 76.10 & 48.96 & \cellcolor{red!25}$+27.1$ \\
\dirdense{} & 92.32 & 50.36 & \cellcolor{red!25}$+42.0$ & 97.52 & 51.24 & \cellcolor{red!25}$\mathbf{+46.3}$ & 80.39 & 24.51 & \cellcolor{red!25}$+55.9$ & 100.00 & 51.42 & \cellcolor{red!25}$+48.6$ & 77.35 & 49.63 & \cellcolor{red!25}$+27.7$ \\
\bottomrule
\end{tabular*}
\end{table}

Tab.~\ref{tab:main} compares deployed readouts with both ADR variants across nine model families on the five direction-balanced subsets. On four of the five sets, both variants achieve large grounded gains across all families, while the deployed scores remain close to chance. ADR recovers direction from the same frozen features without further encoder training. This supports readout blindness as a source of directional failure. Full results across checkpoints and readouts are given in App.~\ref{app:grid}.

What'sUp-A is the main exception. On this set, \dirpix{} offers little improvement over the deployed readout for eight families, while SigLIP2 achieves a grounded gain of $+47.3$ percentage points. The error analysis finds fewer localization failures for SigLIP2 than for CLIP, consistent with this advantage (App.~\ref{app:wualoc}).
On VSR, some remaining errors arise from a reference-frame mismatch, as ADR follows viewer-centered relations while some annotations are object-centered (App.~\ref{app:vsrmain}).

Localization weights also affect performance. \dirdense{} gives larger grounded gains in most comparisons, especially on What'sUp-A, while \dirpix{} generally performs better on VG-Relation left/right. The dense weights $\pi_d$ assign mass to each word independently, without the shared patch budget imposed by OT. This can improve localization when objects differ in size (App.~\ref{app:localizers}). We keep \dirpix{} as the default for the reasons in Sec.~\ref{sec:method} and report both variants to show this trade-off.

\begin{figure}[t]
\centering
\includegraphics[width=\columnwidth]{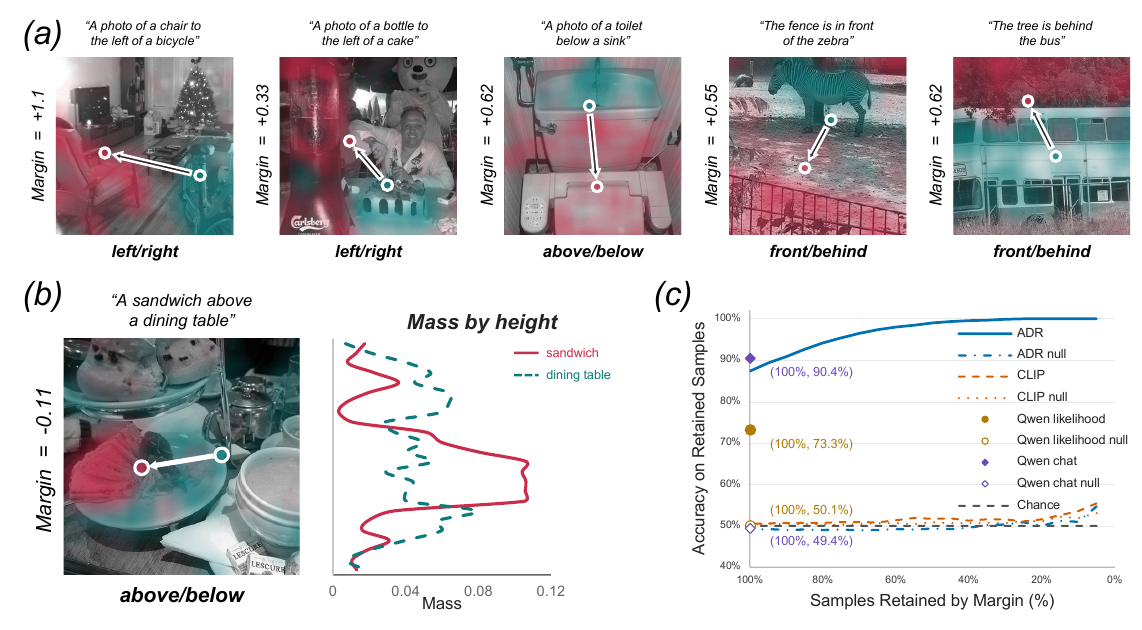}\par
\caption{ADR predictions and reliability. (a) Correct predictions for horizontal, vertical and depth relations. \textcolor[HTML]{C92D4B}{Red} and \textcolor[HTML] {0D7A7C}{blue green} maps show subject and object excess mass. Circles mark centroids, and arrows point from object to subject. Margin is the score difference between correct and incorrect captions, $D(C^+)-D(C^-)$ for ADR. (b) A localization failure of \dirpix{}. Curves show the fraction of each word's excess mass at each image height. The diffuse mass for "dining" produces an incorrect centroid order and prediction. (c) Accuracy on the binary direction-balanced sets after retaining samples with the largest absolute margins. Increasing the margin threshold retains fewer samples but yields higher ADR accuracy, showing that its margin helps select reliable predictions.}

\label{fig:measurement}
\end{figure}

\subsection{Comparison with Alternative Readouts}
\label{sec:alternatives}

Sec.~\ref{sec:impossibility} discusses alternative readouts that partially recover direction. We now compare ADR with these readouts on the five direction-balanced subsets in Tab.~\ref{tab:alternatives}. SigLIP2-based ADR achieves the largest grounded gains among the evaluated non-chat readouts on four of the five sets and is competitive with Qwen3.6 chat. Even with CLIP, ADR outperforms most non-chat alternatives.

Half-crop recovers directional signal by comparing image regions. Cross-encoder matching heads use joint attention to score image-text pairs. MLLMs score answers by likelihood or generate them through chat. All three approaches achieve positive grounded gains on most tasks, and chat further improves Qwen3.6 over likelihood, but has mixed effects on LLaVA-1.5. These gains further support that directional failures can arise from readout blindness rather than missing directional information.

However, half-crop changes the image input and requires multiple image encodings. Matching heads learn image-text compatibility without an explicit rule for comparing object positions, and their gains remain limited on several tasks. MLLM readouts require much more computation than ADR, which adds no learned parameters and little memory overhead. On one H200 GPU, SigLIP2-based \dirpix{} uses about $6\%$ of Qwen3.6 chat's forward FLOPs and $9\%$ of its runtime (Tab.~\ref{tab:cost}).

\subsection{Prediction Reliability and Error Analysis}
\label{sec:reliability}

Fig.~\ref{fig:measurement}(a) shows correct ADR predictions for horizontal, vertical and depth relations. The excess-mass maps and centroids show the estimated object locations and how ADR uses the relative positions to make the prediction. These also help assess how well the features support object localization. The same visualizations show how localization errors affect predictions. Fig.~\ref{fig:measurement}(b) shows a failure to localize the dining table. The excess-mass
map and distribution by image height show diffuse weights for "dining", whose centroid lies above the sandwich centroid. ADR therefore incorrectly predicts that the sandwich is below the dining table. These measurements help identify the mislocalized object and explain how its estimated position leads to the wrong prediction.

Beyond individual cases, the absolute margin also serves as a confidence score. Fig.~\ref{fig:measurement}(c) compares prediction reliability on the four binary sets in Tab.~\ref{tab:main}. Increasing the margin threshold retains fewer samples and raises ADR accuracy. Retaining the $30\%$ of samples with the largest absolute margins gives $99.8\%$ accuracy for CLIP-based \dirpix{}, compared with $51.6\%$ for deployed CLIP. Qwen3.6 likelihood and chat achieve $73.3\%$ and $90.4\%$ accuracy on all samples. The ADR margin therefore helps identify samples with reliable predictions. In practice, a threshold chosen on a validation set could decide which test samples have reliable predictions and which need human review.

\section{Conclusion}
\label{sec:conclusion}

We identify readout blindness as a source of directional failure in VLMs and show why some scoring rules cannot distinguish direction. This analysis leads to ADR, which reads direction from an explicit alignment of frozen features without training or learned parameters. We also introduce an evaluation kit that separates grounded gain from prior effects. ADR improves grounded gain over deployed scores on most tested sets, showing that the frozen features retain direction that an appropriate readout can recover. Future work includes broader multimodal alignment tasks, more complex relations beyond noun pairs, and extensions to a wider range of VLM architectures.

\subsubsection*{Acknowledgments}
This work was supported by the Research Council of Finland through Academy Project 376209. This study was also supported by ELLIS Institute Finland and the School of Electrical Engineering, Aalto University. We acknowledge CSC – IT Center for Science, Finland, for providing access to Roihu and LUMI, owned by the EuroHPC Joint Undertaking and hosted by CSC Finland in collaboration with the LUMI consortium. We also acknowledge the computational resources provided by Aalto Science-IT through the Triton cluster.

\bibliography{bib/string, bib/reference}
\bibliographystyle{iclr2027_conference}

\appendix
\providecommand{\pizero}{\pi_0}
\providecommand{\dirpi}{\textsf{ADR-}\ensuremath{\pi_0}}
\providecommand{\dirargmax}{\textsf{ADR-argmax}}
\providecommand{\tr}{\operatorname{tr}}
\providecommand{\ci}[2]{{\footnotesize[#1,\,#2]}}
\providecommand{\toprule}{\hline}
\providecommand{\midrule}{\hline}
\providecommand{\bottomrule}{\hline}
\providecommand{\addlinespace}{}

\section{Evidence for Readout Blindness and Recovery}
\label{app:tables}

\subsection{Global Cosine under Argument Swaps}
\label{app:cosine-diagnostic}
We examine the global image and caption features of CLIP ViT-B/16 and SigLIP2-So400m on the same $3,918$ VG-Relation left/right items as Tab.~\ref{tab:main}. We retain the original image crops, caption pairs and null assignments. For each pair, we measure the text difference $\|\bt(C^+)-\bt(C^-)\|_2$ and the global margin $\bx(I)^\top[\bt(C^+)-\bt(C^-)]$, using unit-norm features. The recomputed scores reproduce every recorded real and null decision.

\begin{table}[h]
\caption{Global cosine diagnostics on VG-Relation left/right ($n=3918$). Text differences are Euclidean distances between unit-norm caption features. The last two rows restrict the comparison to items correctly ranked by \dirpix{}. Accuracies are percentages.}
\label{tab:cosine-diagnostic}
\centering\small
\begin{tabular*}{\textwidth}{@{\extracolsep{\fill}}lcc@{}}
\toprule
 & CLIP ViT-B/16 & SigLIP2-So400m \\
\midrule
Median text difference & 0.2251 & 0.2494 \\
Pairs with identical text features & 0 & 0 \\
Global score ties & 0 & 0 \\
Global real / null accuracy & 49.64 / 49.69 & 49.95 / 50.54 \\
Global wrong, \dirpix{} correct, item count & 1785 & 1799 \\
\midrule
Median text difference, global wrong and \dirpix{} correct & 0.2264 & 0.2505 \\
Median text difference, both correct & 0.2276 & 0.2510 \\
\bottomrule
\end{tabular*}
\end{table}

Tab.~\ref{tab:cosine-diagnostic} shows that every argument swap changes the text features. The global scores nevertheless remain near chance, with no ties under the $10^{-9}$ tolerance. Among items correctly ranked by \dirpix{}, the median text difference is similar whether the global score is correct or wrong. These observations rule out exact text invariance on these pairs and do not support smaller text differences as a simple explanation for the errors. They do not determine whether direction is lost during image pooling or fails to align with the text difference in the global feature space. ADR succeeds on many of the same items using local features and patch coordinates.

\subsection{Prior Effects Across Evaluation Subsets}
\label{app:prior-all}
Tab.~\ref{tab:prior-all} extends the two examples in Tab.~\ref{tab:prior} to all eight evaluation subsets. It uses the same five readouts throughout. Null accuracies are generally close to random guessing on the five direction-balanced subsets, while the three subsets reported separately show larger deviations. The VG-Relation verb subset is selected from the benchmark's relation labels, while actant swap is the subset name used by VALSE~\citep{parcalabescu2022valse}.

\begin{table}[t]
\caption{Prior effects across five readouts on all eight evaluation subsets. The five direction-balanced subsets are used for the main comparisons, and the three prior-biased subsets are reported separately. Acc., n.acc., dev. and gain follow Tab.~\ref{tab:prior}. Random is accuracy from uniformly random predictions. Accuracies are in percent (\%), and dev. and gain in percentage points.}
\label{tab:prior-all}
\centering
\fontsize{8}{9.5}\selectfont
\setlength{\tabcolsep}{3pt}
\begin{tabular*}{\columnwidth}{@{\extracolsep{\fill}}lclccccc@{}}
\toprule
\multirow{2}{*}{Subset} & \multirow{2}{*}{random} & \multirow{2}{*}{metric} & \multirow{2}{*}{CLIP} & \multirow{2}{*}{NegCLIP} & \multirow{2}{*}{SigLIP2} & LLaVA-1.5 & Qwen3.6 \\
 & & & & & & likelihood & likelihood \\
\midrule
\multicolumn{8}{@{}l}{\emph{Direction-balanced subsets}} \\
\multirow{4}{*}{\shortstack[l]{VG-Rel\\left/right}} & \multirow{4}{*}{50.01} & acc. & 49.64 & 50.26 & 49.95 & 65.80 & 74.76 \\
& & n.acc. & 49.69 & 49.67 & 50.54 & 49.44 & 49.29 \\
& & dev. & $-0.3$ & $-0.3$ & $+0.5$ & $-0.6$ & $-0.7$ \\
& & gain & $-0.1$ & $+0.6$ & $-0.6$ & $+16.4$ & $+25.5$ \\
\midrule
\multirow{4}{*}{\shortstack[l]{COCO-spatial}} & \multirow{4}{*}{49.99} & acc. & 45.30 & 44.80 & 49.01 & 79.21 & 88.12 \\
& & n.acc. & 47.28 & 49.50 & 45.05 & 54.21 & 49.01 \\
& & dev. & $-2.7$ & $-0.5$ & $-4.9$ & $+4.2$ & $-1.0$ \\
& & gain & $-2.0$ & $-4.7$ & $+4.0$ & $+25.0$ & $+39.1$ \\
\midrule
\multirow{4}{*}{\shortstack[l]{What'sUp-A}} & \multirow{4}{*}{25.02} & acc. & 26.96 & 33.33 & 29.17 & 30.39 & 67.89 \\
& & n.acc. & 25.98 & 26.23 & 24.75 & 24.26 & 24.51 \\
& & dev. & $+1.0$ & $+1.2$ & $-0.3$ & $-0.8$ & $-0.5$ \\
& & gain & $+1.0$ & $+7.1$ & $+4.4$ & $+6.1$ & $+43.4$ \\
\midrule
\multirow{4}{*}{\shortstack[l]{What'sUp-B}} & \multirow{4}{*}{50.01} & acc. & 53.67 & 48.31 & 51.69 & 63.56 & 79.09 \\
& & n.acc. & 52.83 & 51.69 & 48.87 & 52.83 & 50.56 \\
& & dev. & $+2.8$ & $+1.7$ & $-1.1$ & $+2.8$ & $+0.6$ \\
& & gain & $+0.8$ & $-3.4$ & $+2.8$ & $+10.7$ & $+28.5$ \\
\midrule
\multirow{4}{*}{\shortstack[l]{VSR\\horizontal}} & \multirow{4}{*}{50.01} & acc. & 53.87 & 52.87 & 54.70 & 55.45 & 61.62 \\
& & n.acc. & 53.21 & 52.79 & 52.96 & 52.96 & 52.96 \\
& & dev. & $+3.2$ & $+2.8$ & $+3.0$ & $+3.0$ & $+3.0$ \\
& & gain & $+0.7$ & $+0.1$ & $+1.7$ & $+2.5$ & $+8.7$ \\
\midrule
\multicolumn{8}{@{}l}{\emph{Prior-biased subsets}} \\
\multirow{4}{*}{\shortstack[l]{VG-Rel\\vertical}} & \multirow{4}{*}{50.02} & acc. & 55.87 & 80.13 & 43.13 & 93.63 & 93.75 \\
& & n.acc. & 47.38 & 69.87 & 27.12 & 84.00 & 86.50 \\
& & dev. & $-2.6$ & $+19.9$ & $-22.9$ & $+34.0$ & $+36.5$ \\
& & gain & $+8.5$ & $+10.3$ & $+16.0$ & $+9.6$ & $+7.3$ \\
\midrule
\multirow{4}{*}{\shortstack[l]{VG-Rel\\verbs}} & \multirow{4}{*}{49.98} & acc. & 67.65 & 92.89 & 50.25 & 99.27 & 98.53 \\
& & n.acc. & 65.20 & 84.80 & 30.15 & 98.04 & 96.09 \\
& & dev. & $+15.2$ & $+34.8$ & $-19.8$ & $+48.1$ & $+46.1$ \\
& & gain & $+2.5$ & $+8.1$ & $+20.1$ & $+1.2$ & $+2.4$ \\
\midrule
\multirow{4}{*}{\shortstack[l]{VALSE\\actant swap}} & \multirow{4}{*}{50.03} & acc. & 72.71 & 85.14 & 62.07 & 93.26 & 93.99 \\
& & n.acc. & 41.31 & 62.80 & 28.35 & 88.51 & 88.41 \\
& & dev. & $-8.7$ & $+12.8$ & $-21.7$ & $+38.5$ & $+38.4$ \\
& & gain & $+31.4$ & $+22.3$ & $+33.7$ & $+4.8$ & $+5.6$ \\
\bottomrule
\end{tabular*}
\end{table}

\subsection{Every Readout on the Five Balanced Sets}
\label{app:grid}
Tab.~\ref{tab:allreadouts} reports the complete comparison on the five direction-balanced subsets, including the additional frozen checkpoints evaluated with the two ADR variants. Alternative readouts are reported in Tab.~\ref{tab:alternatives}.

{\fontsize{7}{9}\selectfont\setlength{\tabcolsep}{0.2pt}
\setlength{\LTleft}{0pt}
\setlength{\LTright}{0pt}
\setlength{\LTcapwidth}{\textwidth}
\begin{longtable}{@{\extracolsep{\fill}}l@{\hspace{2pt}}*{15}{c}@{}}
\caption{Readouts on the five direction-balanced subsets. Each subset reports acc., n.acc., and gain as defined in Tab.~\ref{tab:main}. Accuracies are in percent and gains in percentage points. Grey rows show deployed scores, followed by ADR with four choices of alignment weights (App.~\ref{app:localizers}). VSR exceptions are marked and explained below the table. The random-guess row averages $10{,}000$ uniform draws.}
\label{tab:allreadouts}\\
\toprule
\multirow{2}{*}{Readout} & \multicolumn{3}{c}{VG-Rel L/R} & \multicolumn{3}{c}{COCO} & \multicolumn{3}{c}{WU-A} & \multicolumn{3}{c}{WU-B} & \multicolumn{3}{c}{VSR} \\
\cmidrule(lr){2-4}\cmidrule(lr){5-7}\cmidrule(lr){8-10}\cmidrule(lr){11-13}\cmidrule(lr){14-16}
 & acc. & n.acc. & gain & acc. & n.acc. & gain & acc. & n.acc. & gain & acc. & n.acc. & gain & acc. & n.acc. & gain \\
 & \multicolumn{3}{c}{$n{=}3918$} & \multicolumn{3}{c}{$n{=}404$} & \multicolumn{3}{c}{$n{=}408$} & \multicolumn{3}{c}{$n{=}354$} & \multicolumn{3}{c}{$n{=}1201$} \\
random & 50.01 & 50.01 & 0.0 & 49.99 & 49.99 & 0.0 & 25.02 & 25.02 & 0.0 & 50.01 & 50.01 & 0.0 & 50.01 & 50.01 & 0.0 \\
\midrule
\endfirsthead
\toprule
\multirow{2}{*}{Readout} & \multicolumn{3}{c}{VG-Rel L/R} & \multicolumn{3}{c}{COCO} & \multicolumn{3}{c}{WU-A} & \multicolumn{3}{c}{WU-B} & \multicolumn{3}{c}{VSR} \\
\cmidrule(lr){2-4}\cmidrule(lr){5-7}\cmidrule(lr){8-10}\cmidrule(lr){11-13}\cmidrule(lr){14-16}
 & acc. & n.acc. & gain & acc. & n.acc. & gain & acc. & n.acc. & gain & acc. & n.acc. & gain & acc. & n.acc. & gain \\
\midrule
\endhead
\bottomrule
\endfoot
\bottomrule
\endlastfoot
\multicolumn{16}{@{}l}{\emph{frozen dual encoders, deployed score and ADR}} \\
\textcolor[gray]{0.45}{CLIP ViT-B/16} & \textcolor[gray]{0.45}{49.64} & \textcolor[gray]{0.45}{49.69} & \textcolor[gray]{0.45}{$-0.1$} & \textcolor[gray]{0.45}{45.30} & \textcolor[gray]{0.45}{47.28} & \textcolor[gray]{0.45}{$-2.0$} & \textcolor[gray]{0.45}{26.96} & \textcolor[gray]{0.45}{25.98} & \textcolor[gray]{0.45}{$+1.0$} & \textcolor[gray]{0.45}{53.67} & \textcolor[gray]{0.45}{52.83} & \textcolor[gray]{0.45}{$+0.8$} & \textcolor[gray]{0.45}{53.87} & \textcolor[gray]{0.45}{53.21} & \textcolor[gray]{0.45}{$+0.7$} \\*
\hspace{0.3em}\dirpix{} & 90.45 & 49.03 & $+41.4$ & 92.33 & 51.73 & $+40.6$ & 31.62 & 25.74 & $+5.9$ & 93.22 & 52.26 & $+41.0$ & 73.86 & 47.79 & $+26.1$ \\*
\hspace{0.3em}\dirdense{} & 90.22 & 49.44 & $+40.8$ & 94.06 & 52.72 & $+41.3$ & 41.18 & 25.25 & $+15.9$ & 98.02 & 50.00 & $+48.0$ & 74.94 & 47.88 & $+27.1$ \\*
\hspace{0.3em}\dirpi{} & 90.07 & 47.98 & $+42.1$ & 90.84 & 52.23 & $+38.6$ & 29.66 & 25.00 & $+4.7$ & 91.24 & 51.98 & $+39.3$ & 73.11 & 46.96 & $+26.2$ \\*
\hspace{0.3em}\dirargmax{} & 74.82 & 50.75 & $+24.1$ & 84.65 & 46.66 & $+38.0$ & 52.45 & 25.98 & $+26.5$ & 95.90 & 46.75 & $+49.2$ & 69.15 & 51.08 & $+18.1$ \\
\textcolor[gray]{0.45}{CLIP ViT-B/32} & \textcolor[gray]{0.45}{49.31} & \textcolor[gray]{0.45}{49.18} & \textcolor[gray]{0.45}{$+0.1$} & \textcolor[gray]{0.45}{50.50} & \textcolor[gray]{0.45}{49.26} & \textcolor[gray]{0.45}{$+1.2$} & \textcolor[gray]{0.45}{30.64} & \textcolor[gray]{0.45}{28.19} & \textcolor[gray]{0.45}{$+2.5$} & \textcolor[gray]{0.45}{48.59} & \textcolor[gray]{0.45}{52.26} & \textcolor[gray]{0.45}{$-3.7$} & \textcolor[gray]{0.45}{53.62} & \textcolor[gray]{0.45}{51.54} & \textcolor[gray]{0.45}{$+2.1$} \\*
\hspace{0.3em}\dirpix{} & 88.31 & 48.88 & $+39.4$ & 92.08 & 52.97 & $+39.1$ & 27.45 & 24.51 & $+2.9$ & 93.22 & 49.72 & $+43.5$ & 74.35 & 50.21 & $+24.1$ \\*
\hspace{0.3em}\dirdense{} & 88.51 & 48.75 & $+39.8$ & 93.07 & 54.95 & $+38.1$ & 32.35 & 25.00 & $+7.4$ & 96.61 & 48.31 & $+48.3$ & 74.35 & 48.38 & $+26.0$ \\*
\hspace{0.3em}\dirpi{} & 88.00 & 49.23 & $+38.8$ & 90.59 & 54.70 & $+35.9$ & 26.72 & 25.25 & $+1.5$ & 91.81 & 50.56 & $+41.3$ & 73.52 & 49.46 & $+24.1$ \\*
\hspace{0.3em}\dirargmax{} & 75.66 & 49.97 & $+25.7$ & 81.93 & 54.58 & $+27.4$ & 48.53 & 23.53 & $+25.0$ & 94.35 & 47.46 & $+46.9$ & 66.86 & 51.46 & $+15.4$ \\
\textcolor[gray]{0.45}{CLIP ViT-L/14} & \textcolor[gray]{0.45}{50.54} & \textcolor[gray]{0.45}{49.62} & \textcolor[gray]{0.45}{$+0.9$} & \textcolor[gray]{0.45}{50.25} & \textcolor[gray]{0.45}{49.75} & \textcolor[gray]{0.45}{$+0.5$} & \textcolor[gray]{0.45}{26.47} & \textcolor[gray]{0.45}{24.75} & \textcolor[gray]{0.45}{$+1.7$} & \textcolor[gray]{0.45}{51.98} & \textcolor[gray]{0.45}{51.98} & \textcolor[gray]{0.45}{$0.0$} & \textcolor[gray]{0.45}{52.71} & \textcolor[gray]{0.45}{53.71} & \textcolor[gray]{0.45}{$-1.0$} \\*
\hspace{0.3em}\dirpix{} & 90.61 & 50.03 & $+40.6$ & 95.05 & 52.48 & $+42.6$ & 32.35 & 22.55 & $+9.8$ & 100.00 & 49.72 & $+50.3$ & 76.02 & 49.13 & $+26.9$ \\*
\hspace{0.3em}\dirdense{} & 90.99 & 50.10 & $+40.9$ & 95.30 & 52.97 & $+42.3$ & 58.09 & 22.30 & $+35.8$ & 100.00 & 48.87 & $+51.1$ & 76.27 & 49.04 & $+27.2$ \\*
\hspace{0.3em}\dirpi{} & 90.76 & 49.82 & $+40.9$ & 93.56 & 53.22 & $+40.3$ & 30.15 & 21.57 & $+8.6$ & 99.72 & 48.87 & $+50.9$ & 75.35 & 49.38 & $+26.0$ \\*
\hspace{0.3em}\dirargmax{} & 81.55 & 50.24 & $+31.3$ & 88.12 & 52.72 & $+35.4$ & 50.98 & 20.16 & $+30.8$ & 99.58 & 46.47 & $+53.1$ & 71.48 & 51.21 & $+20.3$ \\
\textcolor[gray]{0.45}{LAION-2B ViT-B/16} & \textcolor[gray]{0.45}{49.13} & \textcolor[gray]{0.45}{49.77} & \textcolor[gray]{0.45}{$-0.6$} & \textcolor[gray]{0.45}{43.07} & \textcolor[gray]{0.45}{47.28} & \textcolor[gray]{0.45}{$-4.2$} & \textcolor[gray]{0.45}{27.70} & \textcolor[gray]{0.45}{26.23} & \textcolor[gray]{0.45}{$+1.5$} & \textcolor[gray]{0.45}{50.00} & \textcolor[gray]{0.45}{50.28} & \textcolor[gray]{0.45}{$-0.3$} & \textcolor[gray]{0.45}{51.96} & \textcolor[gray]{0.45}{52.96} & \textcolor[gray]{0.45}{$-1.0$} \\*
\hspace{0.3em}\dirpix{} & 89.05 & 50.23 & $+38.8$ & 91.83 & 51.73 & $+40.1$ & 31.86 & 23.28 & $+8.6$ & 97.74 & 52.26 & $+45.5$ & 74.19 & 48.96 & $+25.2$ \\*
\hspace{0.3em}\dirdense{} & 88.41 & 50.13 & $+38.3$ & 93.56 & 51.24 & $+42.3$ & 38.73 & 25.25 & $+13.5$ & 99.15 & 52.54 & $+46.6$ & 74.94 & 50.04 & $+24.9$ \\*
\hspace{0.3em}\dirpi{} & 88.59 & 50.05 & $+38.5$ & 90.84 & 51.49 & $+39.4$ & 31.62 & 23.53 & $+8.1$ & 96.04 & 50.57 & $+45.5$ & 74.10 & 49.29 & $+24.8$ \\*
\hspace{0.3em}\dirargmax{} & 73.43 & 50.86 & $+22.6$ & 81.56 & 45.30 & $+36.3$ & 43.32 & 25.86 & $+17.5$ & 94.07 & 47.03 & $+47.0$ & 67.57 & 50.92 & $+16.7$ \\
\textcolor[gray]{0.45}{DFN5B ViT-H/14} & \textcolor[gray]{0.45}{50.46} & \textcolor[gray]{0.45}{50.89} & \textcolor[gray]{0.45}{$-0.4$} & \textcolor[gray]{0.45}{44.80} & \textcolor[gray]{0.45}{48.51} & \textcolor[gray]{0.45}{$-3.7$} & \textcolor[gray]{0.45}{28.92} & \textcolor[gray]{0.45}{27.21} & \textcolor[gray]{0.45}{$+1.7$} & \textcolor[gray]{0.45}{53.11} & \textcolor[gray]{0.45}{51.69} & \textcolor[gray]{0.45}{$+1.4$} & \textcolor[gray]{0.45}{54.04} & \textcolor[gray]{0.45}{52.21} & \textcolor[gray]{0.45}{$+1.8$} \\*
\hspace{0.3em}\dirpix{} & 91.22 & 47.93 & $+43.3$ & 97.28 & 50.99 & $+46.3$ & 26.96 & 25.49 & $+1.5$ & 98.59 & 51.13 & $+47.5$ & 75.85 & 48.46 & $+27.4$ \\*
\hspace{0.3em}\dirdense{} & 91.88 & 48.95 & $+42.9$ & 98.27 & 51.24 & $+47.0$ & 63.24 & 24.75 & $+38.5$ & 100.00 & 51.69 & $+48.3$ & 77.10 & 46.96 & $+30.1$ \\*
\hspace{0.3em}\dirpi{} & 90.91 & 47.52 & $+43.4$ & 95.30 & 50.50 & $+44.8$ & 25.00 & 25.74 & $-0.7$ & 97.18 & 51.98 & $+45.2$ & 74.60 & 46.96 & $+27.6$ \\*
\hspace{0.3em}\dirargmax{} & 83.59 & 51.10 & $+32.5$ & 93.56 & 54.21 & $+39.4$ & 54.29 & 25.49 & $+28.8$ & 97.18 & 48.30 & $+48.9$ & 73.90 & 50.12 & $+23.8$ \\
\textcolor[gray]{0.45}{MetaCLIP ViT-H/14} & \textcolor[gray]{0.45}{50.33} & \textcolor[gray]{0.45}{48.72} & \textcolor[gray]{0.45}{$+1.6$} & \textcolor[gray]{0.45}{48.51} & \textcolor[gray]{0.45}{49.75} & \textcolor[gray]{0.45}{$-1.2$} & \textcolor[gray]{0.45}{27.70} & \textcolor[gray]{0.45}{23.28} & \textcolor[gray]{0.45}{$+4.4$} & \textcolor[gray]{0.45}{51.41} & \textcolor[gray]{0.45}{49.44} & \textcolor[gray]{0.45}{$+2.0$} & \textcolor[gray]{0.45}{54.62} & \textcolor[gray]{0.45}{53.96} & \textcolor[gray]{0.45}{$+0.7$} \\*
\hspace{0.3em}\dirpix{} & 90.76 & 49.34 & $+41.4$ & 92.57 & 48.51 & $+44.1$ & 26.72 & 25.25 & $+1.5$ & 93.50 & 46.33 & $+47.2$ & 73.86 & 48.88 & $+25.0$ \\*
\hspace{0.3em}\dirdense{} & 91.04 & 49.62 & $+41.4$ & 94.80 & 48.02 & $+46.8$ & 52.70 & 25.00 & $+27.7$ & 100.00 & 51.42 & $+48.6$ & 74.85 & 49.71 & $+25.1$ \\*
\hspace{0.3em}\dirpi{} & 90.28 & 49.67 & $+40.6$ & 91.58 & 49.26 & $+42.3$ & 25.00 & 25.49 & $-0.5$ & 78.25 & 47.18 & $+31.1$ & 72.86 & 49.46 & $+23.4$ \\*
\hspace{0.3em}\dirargmax{} & 79.43 & 49.95 & $+29.5$ & 85.40 & 50.99 & $+34.4$ & 46.45 & 24.51 & $+21.9$ & 96.04 & 48.02 & $+48.0$ & 69.48 & 49.29 & $+20.2$ \\
\textcolor[gray]{0.45}{SigLIP ViT-B/16} & \textcolor[gray]{0.45}{50.33} & \textcolor[gray]{0.45}{49.31} & \textcolor[gray]{0.45}{$+1.0$} & \textcolor[gray]{0.45}{48.51} & \textcolor[gray]{0.45}{45.79} & \textcolor[gray]{0.45}{$+2.7$} & \textcolor[gray]{0.45}{27.94} & \textcolor[gray]{0.45}{23.77} & \textcolor[gray]{0.45}{$+4.2$} & \textcolor[gray]{0.45}{49.15} & \textcolor[gray]{0.45}{47.74} & \textcolor[gray]{0.45}{$+1.4$} & \textcolor[gray]{0.45}{54.70} & \textcolor[gray]{0.45}{52.29} & \textcolor[gray]{0.45}{$+2.4$} \\*
\hspace{0.3em}\dirpix{} & 90.02 & 49.62 & $+40.4$ & 89.36 & 47.52 & $+41.8$ & 54.17 & 23.28 & $+30.9$ & 100.00 & 50.56 & $+49.4$ & 74.27 & 49.54 & $+24.7$ \\*
\hspace{0.3em}\dirdense{} & 90.28 & 49.49 & $+40.8$ & 91.58 & 49.75 & $+41.8$ & 56.86 & 21.32 & $+35.5$ & 100.00 & 52.83 & $+47.2$ & 74.19 & 49.54 & $+24.7$ \\*
\hspace{0.3em}\dirpi{} & 89.84 & 49.46 & $+40.4$ & 88.12 & 49.26 & $+38.9$ & 47.79 & 22.06 & $+25.7$ & 100.00 & 50.85 & $+49.2$ & 73.44 & 47.71 & $+25.7$ \\*
\hspace{0.3em}\dirargmax{} & 79.81 & 49.44 & $+30.4$ & 81.81 & 53.34 & $+28.5$ & 53.68 & 27.14 & $+26.5$ & 90.26 & 49.30 & $+41.0$ & 70.77 & 50.58 & $+20.2$ \\
\textcolor[gray]{0.45}{SigLIP-So400m} & \textcolor[gray]{0.45}{50.79} & \textcolor[gray]{0.45}{51.23} & \textcolor[gray]{0.45}{$-0.4$} & \textcolor[gray]{0.45}{49.01} & \textcolor[gray]{0.45}{53.96} & \textcolor[gray]{0.45}{$-5.0$} & \textcolor[gray]{0.45}{27.21} & \textcolor[gray]{0.45}{22.79} & \textcolor[gray]{0.45}{$+4.4$} & \textcolor[gray]{0.45}{51.98} & \textcolor[gray]{0.45}{54.80} & \textcolor[gray]{0.45}{$-2.8$} & \textcolor[gray]{0.45}{55.04} & \textcolor[gray]{0.45}{53.12} & \textcolor[gray]{0.45}{$+1.9$} \\*
\hspace{0.3em}\dirpix{} & 91.50 & 50.26 & $+41.2$ & 94.55 & 48.76 & $+45.8$ & 66.67 & 27.21 & $+39.5$ & 100.00 & 55.37 & $+44.6$ & 76.44 & 47.88 & $+28.6$ \\*
\hspace{0.3em}\dirdense{} & 91.83 & 50.66 & $+41.2$ & 97.28 & 50.25 & $+47.0$ & 71.81 & 29.66 & $+42.2$ & 100.00 & 53.11 & $+46.9$ & 77.27 & 47.46 & $+29.8$ \\*
\hspace{0.3em}\dirpi{} & 91.30 & 50.84 & $+40.5$ & 94.55 & 49.50 & $+45.1$ & 58.09 & 28.43 & $+29.7$ & 100.00 & 53.95 & $+46.1$ & 76.44 & 47.54 & $+28.9$ \\*
\hspace{0.3em}\dirargmax{} & 82.50 & 50.13 & $+32.4$ & 87.13 & 51.36 & $+35.8$ & 49.75 & 24.02 & $+25.7$ & 96.05 & 51.41 & $+44.6$ & 72.31 & 48.96 & $+23.4$ \\
\textcolor[gray]{0.45}{SigLIP2-So400m} & \textcolor[gray]{0.45}{49.95} & \textcolor[gray]{0.45}{50.54} & \textcolor[gray]{0.45}{$-0.6$} & \textcolor[gray]{0.45}{49.01} & \textcolor[gray]{0.45}{45.05} & \textcolor[gray]{0.45}{$+4.0$} & \textcolor[gray]{0.45}{29.17} & \textcolor[gray]{0.45}{24.75} & \textcolor[gray]{0.45}{$+4.4$} & \textcolor[gray]{0.45}{51.69} & \textcolor[gray]{0.45}{48.87} & \textcolor[gray]{0.45}{$+2.8$} & \textcolor[gray]{0.45}{54.70} & \textcolor[gray]{0.45}{52.96} & \textcolor[gray]{0.45}{$+1.7$} \\*
\hspace{0.3em}\dirpix{} & 92.09 & 49.46 & $+42.6$ & 96.53 & 52.48 & $+44.1$ & 74.26 & 26.96 & $+47.3$ & 100.00 & 50.85 & $+49.2$ & 76.10 & 48.96 & $+27.1$ \\*
\hspace{0.3em}\dirdense{} & 92.32 & 50.36 & $+42.0$ & 97.52 & 51.24 & $+46.3$ & 80.39 & 24.51 & $+55.9$ & 100.00 & 51.42 & $+48.6$ & 77.35 & 49.63 & $+27.7$ \\*
\hspace{0.3em}\dirpi{} & 91.68 & 50.43 & $+41.3$ & 94.55 & 50.50 & $+44.1$ & 72.30 & 24.75 & $+47.6$ & 100.00 & 51.41 & $+48.6$ & 75.94 & 49.88 & $+26.1$ \\*
\hspace{0.3em}\dirargmax{} & 81.46 & 48.90 & $+32.6$ & 88.49 & 49.38 & $+39.1$ & 43.75 & 24.14 & $+19.6$ & 93.64 & 49.15 & $+44.5$ & 71.61 & 50.58 & $+21.0$ \\
\textcolor[gray]{0.45}{BLIP ViT-B/16, ITC} & \textcolor[gray]{0.45}{49.31} & \textcolor[gray]{0.45}{49.00} & \textcolor[gray]{0.45}{$+0.3$} & \textcolor[gray]{0.45}{45.79} & \textcolor[gray]{0.45}{47.52} & \textcolor[gray]{0.45}{$-1.7$} & \textcolor[gray]{0.45}{27.70} & \textcolor[gray]{0.45}{25.74} & \textcolor[gray]{0.45}{$+2.0$} & \textcolor[gray]{0.45}{52.82} & \textcolor[gray]{0.45}{49.43} & \textcolor[gray]{0.45}{$+3.4$} & \textcolor[gray]{0.45}{52.96} & \textcolor[gray]{0.45}{52.21} & \textcolor[gray]{0.45}{$+0.8$} \\*
\hspace{0.3em}\dirpix{} & 89.51 & 48.47 & $+41.0$ & 90.59 & 52.72 & $+37.9$ & 27.70 & 27.94 & $-0.2$ & 96.61 & 51.13 & $+45.5$ & 72.94 & 49.63 & $+23.3$ \\*
\hspace{0.3em}\dirdense{} & 88.97 & 49.11 & $+39.9$ & 95.79 & 51.24 & $+44.6$ & 41.91 & 25.00 & $+16.9$ & 100.00 & 50.56 & $+49.4$ & 75.69 & 48.79 & $+26.9$ \\*
\hspace{0.3em}\dirpi{} & 89.13 & 48.65 & $+40.5$ & 90.84 & 50.50 & $+40.3$ & 27.45 & 28.92 & $-1.5$ & 96.61 & 50.00 & $+46.6$ & 73.27 & 49.38 & $+23.9$ \\*
\hspace{0.3em}\dirargmax{} & 80.45 & 50.68 & $+29.8$ & 89.48 & 53.59 & $+35.9$ & 35.85 & 29.53 & $+6.3$ & 99.44 & 50.28 & $+49.2$ & 70.94 & 48.13 & $+22.8$ \\
\textcolor[gray]{0.45}{NegCLIP} & \textcolor[gray]{0.45}{50.26} & \textcolor[gray]{0.45}{49.67} & \textcolor[gray]{0.45}{$+0.6$} & \textcolor[gray]{0.45}{44.80} & \textcolor[gray]{0.45}{49.50} & \textcolor[gray]{0.45}{$-4.7$} & \textcolor[gray]{0.45}{33.33} & \textcolor[gray]{0.45}{26.23} & \textcolor[gray]{0.45}{$+7.1$} & \textcolor[gray]{0.45}{48.31} & \textcolor[gray]{0.45}{51.69} & \textcolor[gray]{0.45}{$-3.4$} & \textcolor[gray]{0.45}{52.87} & \textcolor[gray]{0.45}{52.79} & \textcolor[gray]{0.45}{$+0.1$} \\*
\hspace{0.3em}\dirpix{} & 89.13 & 48.65 & $+40.5$ & 91.83 & 53.71 & $+38.1$ & 28.68 & 24.26 & $+4.4$ & 93.79 & 47.18 & $+46.6$ & 73.94 & 49.54 & $+24.4$ \\*
\hspace{0.3em}\dirdense{} & 89.13 & 49.18 & $+40.0$ & 93.07 & 52.97 & $+40.1$ & 33.82 & 25.00 & $+8.8$ & 96.89 & 49.15 & $+47.7$ & 74.10 & 49.46 & $+24.6$ \\*
\hspace{0.3em}\dirpi{} & 89.25 & 48.55 & $+40.7$ & 91.34 & 54.70 & $+36.6$ & 27.21 & 25.00 & $+2.2$ & 92.09 & 48.58 & $+43.5$ & 73.61 & 49.88 & $+23.7$ \\*
\hspace{0.3em}\dirargmax{} & 79.42 & 49.77 & $+29.7$ & 85.15 & 54.21 & $+30.9$ & 50.55 & 25.98 & $+24.6$ & 96.33 & 44.49 & $+51.8$ & 68.57 & 51.79 & $+16.8$ \\
\textcolor[gray]{0.45}{CLoVe} & \textcolor[gray]{0.45}{48.83} & \textcolor[gray]{0.45}{50.05} & \textcolor[gray]{0.45}{$-1.2$} & \textcolor[gray]{0.45}{51.49} & \textcolor[gray]{0.45}{50.99} & \textcolor[gray]{0.45}{$+0.5$} & \textcolor[gray]{0.45}{34.56} & \textcolor[gray]{0.45}{23.77} & \textcolor[gray]{0.45}{$+10.8$} & \textcolor[gray]{0.45}{48.59} & \textcolor[gray]{0.45}{47.74} & \textcolor[gray]{0.45}{$+0.9$} & \textcolor[gray]{0.45}{50.96} & \textcolor[gray]{0.45}{51.79} & \textcolor[gray]{0.45}{$-0.8$} \\*
\hspace{0.3em}\dirpix{} & 88.57 & 48.29 & $+40.3$ & 90.84 & 55.45 & $+35.4$ & 30.64 & 25.25 & $+5.4$ & 95.20 & 51.41 & $+43.8$ & 74.69 & 48.96 & $+25.7$ \\*
\hspace{0.3em}\dirdense{} & 88.51 & 48.77 & $+39.7$ & 93.32 & 52.72 & $+40.6$ & 38.24 & 26.23 & $+12.0$ & 99.43 & 51.13 & $+48.3$ & 75.52 & 49.29 & $+26.2$ \\*
\hspace{0.3em}\dirpi{} & 88.51 & 49.21 & $+39.3$ & 89.85 & 54.21 & $+35.6$ & 29.41 & 25.74 & $+3.7$ & 93.79 & 51.69 & $+42.1$ & 75.27 & 49.38 & $+25.9$ \\*
\hspace{0.3em}\dirargmax{} & 78.41 & 48.52 & $+29.9$ & 87.00 & 48.64 & $+38.4$ & 33.88 & 25.86 & $+8.0$ & 98.17 & 52.54 & $+45.6$ & 71.94 & 48.33 & $+23.6$ \\
\textcolor[gray]{0.45}{CE-CLIP} & \textcolor[gray]{0.45}{51.30} & \textcolor[gray]{0.45}{50.94} & \textcolor[gray]{0.45}{$+0.4$} & \textcolor[gray]{0.45}{52.48} & \textcolor[gray]{0.45}{50.74} & \textcolor[gray]{0.45}{$+1.7$} & \textcolor[gray]{0.45}{31.13} & \textcolor[gray]{0.45}{27.45} & \textcolor[gray]{0.45}{$+3.7$} & \textcolor[gray]{0.45}{50.00} & \textcolor[gray]{0.45}{49.44} & \textcolor[gray]{0.45}{$+0.6$} & \textcolor[gray]{0.45}{53.29} & \textcolor[gray]{0.45}{52.71} & \textcolor[gray]{0.45}{$+0.6$} \\*
\hspace{0.3em}\dirpix{} & 90.20 & 48.21 & $+42.0$ & 92.82 & 55.69 & $+37.1$ & 30.15 & 25.00 & $+5.2$ & 92.09 & 48.87 & $+43.2$ & 74.85 & 49.46 & $+25.4$ \\*
\hspace{0.3em}\dirdense{} & 89.84 & 48.34 & $+41.5$ & 91.83 & 54.70 & $+37.1$ & 31.13 & 25.25 & $+5.9$ & 92.94 & 49.72 & $+43.2$ & 74.77 & 49.21 & $+25.6$ \\*
\hspace{0.3em}\dirpi{} & 89.92 & 48.70 & $+41.2$ & 90.84 & 55.45 & $+35.4$ & 28.92 & 24.75 & $+4.2$ & 90.39 & 49.44 & $+41.0$ & 74.69 & 49.29 & $+25.4$ \\*
\hspace{0.3em}\dirargmax{} & 76.75 & 49.08 & $+27.7$ & 87.00 & 53.71 & $+33.3$ & 49.75 & 24.08 & $+25.7$ & 87.85 & 47.17 & $+40.7$ & 71.32 & 49.54 & $+21.8$ \\
\textcolor[gray]{0.45}{direction-targeted recipe} & \textcolor[gray]{0.45}{50.38} & \textcolor[gray]{0.45}{48.93} & \textcolor[gray]{0.45}{$+1.5$} & \textcolor[gray]{0.45}{51.24} & \textcolor[gray]{0.45}{48.02} & \textcolor[gray]{0.45}{$+3.2$} & \textcolor[gray]{0.45}{31.86} & \textcolor[gray]{0.45}{25.49} & \textcolor[gray]{0.45}{$+6.4$} & \textcolor[gray]{0.45}{62.14} & \textcolor[gray]{0.45}{47.74} & \textcolor[gray]{0.45}{$+14.4$} & \textcolor[gray]{0.45}{53.96} & \textcolor[gray]{0.45}{52.37} & \textcolor[gray]{0.45}{$+1.6$} \\*
\hspace{0.3em}\dirpix{} & 86.73 & 48.42 & $+38.3$ & 88.86 & 53.71 & $+35.2$ & 26.72 & 26.23 & $+0.5$ & 78.24 & 47.74 & $+30.5$ & 71.86 & 50.29 & $+21.6$ \\*
\hspace{0.3em}\dirdense{} & 87.19 & 48.60 & $+38.6$ & 89.11 & 54.70 & $+34.4$ & 34.56 & 25.49 & $+9.1$ & 89.55 & 49.15 & $+40.4$ & 74.77 & 51.79 & $+23.0$ \\*
\hspace{0.3em}\dirpi{} & 86.83 & 48.77 & $+38.1$ & 87.87 & 53.96 & $+33.9$ & 25.98 & 25.74 & $+0.2$ & 76.27 & 48.58 & $+27.7$ & 72.61 & 50.46 & $+22.2$ \\*
\hspace{0.3em}\dirargmax{} & 79.02 & 47.52 & $+31.5$ & 82.92 & 53.09 & $+29.8$ & 48.41 & 23.90 & $+24.5$ & 87.71 & 49.30 & $+38.4$ & 72.02 & 51.67 & $+20.4$ \\
\midrule
\multicolumn{16}{@{}l}{\emph{matching heads and their contrastive scores}} \\
BLIP base, ITM & 51.02 & 50.79 & $+0.2$ & 55.94 & 48.27 & $+7.7$ & 35.78 & 26.23 & $+9.6$ & 55.37 & 49.15 & $+6.2$ & 54.29 & 50.96 & $+3.3$ \\
BLIP base, ITC & 49.31 & 49.00 & $+0.3$ & 45.79 & 47.52 & $-1.7$ & 27.70 & 25.74 & $+2.0$ & 52.82 & 49.43 & $+3.4$ & 53.02 & 52.04$^{\dagger}$ & $+1.0$ \\
BLIP large, ITM & 50.38 & 49.62 & $+0.8$ & 55.20 & 47.77 & $+7.4$ & 31.74 & 24.51 & $+7.2$ & 60.45 & 46.89 & $+13.6$ & 54.29 & 50.96 & $+3.3$ \\
BLIP large, ITC & 50.43 & 49.92 & $+0.5$ & 41.83 & 50.74 & $-8.9$ & 28.92 & 25.49 & $+3.4$ & 54.23 & 46.89 & $+7.3$ & 53.02 & 51.88$^{\dagger}$ & $+1.1$ \\
BLIP-2 ViT-g, ITM & 49.90 & 49.36 & $+0.5$ & 58.66 & 50.00 & $+8.7$ & 44.61 & 21.32 & $+23.3$ & 59.32 & 48.87 & $+10.5$ & 54.04 & 51.37 & $+2.7$ \\
BLIP-2 ViT-g, ITC & 49.62 & 49.46 & $+0.2$ & 47.77 & 49.50 & $-1.7$ & 32.11 & 25.98 & $+6.1$ & 51.42 & 49.15 & $+2.3$ & 53.59 & 51.14$^{\dagger}$ & $+2.5$ \\
FLAVA-full, ITM & 50.28 & 50.87 & $-0.6$ & 54.70 & 51.49 & $+3.2$ & 30.15 & 26.47 & $+3.7$ & 51.13 & 48.02 & $+3.1$ & 52.29 & 50.87 & $+1.4$ \\
FLAVA-full, ITC & 49.62 & 50.26 & $-0.6$ & 50.74 & 50.00 & $+0.7$ & 28.43 & 26.23 & $+2.2$ & 50.00 & 51.13 & $-1.1$ & 52.29 & 50.98$^{\dagger}$ & $+1.3$ \\
\midrule
\multicolumn{16}{@{}l}{\emph{input intervention}} \\
half-crop, CLIP & 83.67 & 49.16 & $+34.5$ & 82.92 & 55.45 & $+27.5$ & 44.61 & 22.30 & $+22.3$ & 100.00 & 50.00 & $+50.0$ & 60.95 & 52.96$^{\ddagger}$ & $+8.0$ \\
half-crop, SigLIP2 & 88.46 & 49.69 & $+38.8$ & 91.83 & 49.50 & $+42.3$ & 57.35 & 25.25 & $+32.1$ & 99.72 & 51.70 & $+48.0$ & 66.44 & 51.46$^{\ddagger}$ & $+15.0$ \\
\midrule
\multicolumn{16}{@{}l}{\emph{MLLM readouts}} \\
LLaVA-1.5-7B, likelihood & 65.80 & 49.44 & $+16.4$ & 79.21 & 54.21 & $+25.0$ & 30.39 & 24.26 & $+6.1$ & 63.56 & 52.83 & $+10.7$ & 55.45 & 52.96 & $+2.5$ \\
LLaVA-1.5-7B, chat & 65.74 & 50.10 & $+15.6$ & 65.59 & 50.74 & $+14.9$ & 44.00 & 21.45 & $+22.6$ & 65.11 & 51.41 & $+13.7$ & 63.70 & 50.04 & $+13.7$ \\
Qwen3-VL-8B, likelihood & 74.86 & 48.98 & $+25.9$ & 83.66 & 51.24 & $+32.4$ & 51.72 & 26.47 & $+25.3$ & 75.99 & 52.26 & $+23.7$ & 57.79 & 52.37 & $+5.4$ \\
Qwen3-VL-8B, chat & 87.93 & 49.69 & $+38.2$ & 88.37 & 49.57 & $+38.8$ & 99.42 & 25.89 & $+73.5$ & 100.00 & 49.22 & $+50.8$ & 83.76 & 50.21 & $+33.6$ \\
Qwen3.6-27B, likelihood & 74.76 & 49.29 & $+25.5$ & 88.12 & 49.01 & $+39.1$ & 67.89 & 24.51 & $+43.4$ & 79.09 & 50.56 & $+28.5$ & 61.62 & 52.96 & $+8.7$ \\
Qwen3.6-27B, chat & 89.74 & 49.07 & $+40.7$ & 96.16 & 51.11 & $+45.1$ & 99.39 & 32.97 & $+66.4$ & 100.00 & 50.00 & $+50.0$ & 88.01 & 49.71 & $+38.3$ \\
\end{longtable}
\noindent $^{\dagger}$These VSR cells use all $1224$ horizontal items. The other VSR cells use the $1201$ covered items. $^{\ddagger}$These half-crop VSR results use an oracle threshold on the raw caption score. Tab.~\ref{tab:alternatives} instead reports the sign of the score difference between a caption and its swap. Other scoring rules follow App.~\ref{app:scoring}.
}

\section{Theory: Statements and Proofs}
\label{app:proofs}

\subsection{Relational Score}
\label{app:relation-score}

The coupling entry $\pi_{ij}$ describes the alignment between patch $i$ and word $j$. For words $j$ and $j'$, we weight the image relation $A_{ii'}$ by $\pi_{ij}\pi_{i'j'}$, accounting for the alignment of both patches. Summing over patches gives
\[
(\pi^\top\bA\pi)_{jj'}
=
\sum_{i,i'}\pi_{ij}A_{ii'}\pi_{i'j'}.
\]
We combine these aligned image relations with the text relations using
\[
\sum_{j,j'}(\pi^\top\bA\pi)_{jj'}B_{j'j}
=
\operatorname{tr}(\pi^\top\bA\pi\bB).
\]
This is the relational score studied in the main text.

\subsection{Unary Blindness}
\label{app:unary-proof}

\begin{proof}[Proof of Theorem~\ref{thm:unary}]
Each word is encoded independently, so equal content-word multisets give equal feature multisets. A unary score receives only this multiset and the image. Its inputs are therefore identical for $C^+$ and $C^-$, which gives $F(I,C^+)=F(I,C^-)$ for every image and every choice of word features.
\end{proof}

\subsection{Symmetry Decomposition}
\label{app:decomposition}

\begin{proof}[Proof of Eq.~\ref{eq:decomposition}]
Expanding $\bA=\bA^s+\bA^a$ and $\bB=\bB^s+\bB^a$ gives four trace terms. For one mixed term, invariance of the trace under transposition and cyclic permutation gives
\begin{align*}
\operatorname{tr}(\pi^\top\bA^s\pi\bB^a)
 &=\operatorname{tr}\big((\bB^a)^\top\pi^\top(\bA^s)^\top\pi\big)\\
 &=-\operatorname{tr}(\bB^a\pi^\top\bA^s\pi)
 =-\operatorname{tr}(\pi^\top\bA^s\pi\bB^a).
\end{align*}
This term is zero. The same calculation with $\bA^a$ and $\bB^s$ removes the other mixed term. The two remaining terms give the stated decomposition for any fixed coupling.
\end{proof}

\subsection{Symmetric Blindness}
\label{app:symmetry-proof}

\begin{proof}[Proof of Theorem~\ref{thm:relational}]
If $\bA$ is symmetric, then $\bA^a=0$. Eq.~\ref{eq:decomposition} reduces the relational score to $\operatorname{tr}(\pi^\top\bA\pi\bB^s)$. Transposing $\bB$ preserves $\bB^s$, so it preserves the score. In Eq.~\ref{eq:oneshot}, fixing $\bS$ also fixes $\pi_0$ for the chosen $\varepsilon$, and the unary term is independent of $\bB$. The entire score is therefore invariant for any mixing weight $\alpha$.
\end{proof}

\subsection{Argument Swaps and Chance Accuracy}
\label{app:swap-proof}

\begin{proof}[Proof of Theorem~\ref{thm:swap}]
A parse-isomorphic swap preserves the word multiset, so equality for unary scores follows from Theorem~\ref{thm:unary}. For Eq.~\ref{eq:oneshot}, let $\bA$ be symmetric as in Theorem~\ref{thm:swap}. Let $\bP$ be the permutation matrix exchanging the two noun positions, and use superscripts $+$ and $-$ for quantities computed from $C^+$ and $C^-$. Context-free word encoding gives
\[
\bS^-=\bS^+\bP.
\]
The uniform word marginal makes $\pi\mapsto\pi\bP$ a bijection of the feasible couplings in Eq.~\ref{eq:ot}. Both the similarity term and the relative entropy are preserved when the columns of $\bS$ and $\pi$ are permuted together. Since $\varepsilon>0$ makes the objective strictly concave, its unique maximizer obeys
\[
\pi_0^-=\pi_0^+\bP.
\]
The assumption that the rows and columns of $\bB^s$ follow the noun permutation gives
\[
(\bB^-)^s=\bP^\top(\bB^+)^s\bP.
\]
Only these symmetric relations contribute when $\bA$ is symmetric, by Eq.~\ref{eq:decomposition}. Substituting the permutation identities into each term of Eq.~\ref{eq:oneshot}, and using $\bP\bP^\top=\bI$, gives
\begin{align*}
\langle\bS^-,\pi_0^-\rangle
 &=\operatorname{tr}\big(\bP^\top(\bS^+)^\top\pi_0^+\bP\big)
 =\langle\bS^+,\pi_0^+\rangle,\\
\operatorname{tr}\big((\pi_0^-)^\top\bA\pi_0^-(\bB^-)^s\big)
 &=\operatorname{tr}\big(\bP^\top(\pi_0^+)^\top\bA\pi_0^+(\bB^+)^s\bP\big)\\
 &=\operatorname{tr}\big((\pi_0^+)^\top\bA\pi_0^+(\bB^+)^s\big).
\end{align*}
Thus $F(I,C^+)=F(I,C^-)$ for every such item. Assigning a tie the value $1/2$ yields exactly chance accuracy on the binary task.
\end{proof}

\subsection{Centroid Displacement Identity}
\label{app:centroid-proof}

\begin{proof}[Proof of Proposition~\ref{prop:centroid}]
By definition, $\sum_i\pi_{ij}=\mu_j$ and $\sum_i\pi_{ij}\bp_i=\mu_j\bc_j$. The word-relation matrix has only two nonzero entries, $B^a_{os}=1$ and $B^a_{so}=-1$. Expanding the trace therefore gives
\begin{align*}
\operatorname{tr}(\pi^\top\bA^a\pi\bB^a)
 &=\sum_{i,i'}\pi_{is}\pi_{i'o}\,\be^\top(\bp_i-\bp_{i'})
   -\sum_{i,i'}\pi_{io}\pi_{i's}\,\be^\top(\bp_i-\bp_{i'})\\
 &=\mu_s\mu_o\,\be^\top(\bc_s-\bc_o)
   -\mu_o\mu_s\,\be^\top(\bc_o-\bc_s)\\
 &=2\mu_s\mu_o\,\be^\top(\bc_s-\bc_o).
\end{align*}
Each double sum factors into one column mass and one weighted position sum. Positive subject and object masses ensure that both centroids are defined and that the factor $2\mu_s\mu_o$ preserves the sign of their projected displacement.
\end{proof}

\subsection{The Zero Threshold}
\label{app:zero}

The additive gauge $\bB\mapsto \bB+\gamma\,\mathbf{1}\mathbf{1}^{\top}$ shifts the relational term of a symmetric-kernel score by $\gamma\,\mathbf{a}^{\top}\bA\,\mathbf{a}$ for every $\pi\in U(\mathbf{a},\mathbf{b})$, independent of the caption, so the symmetric channel has no canonical threshold. The constant matrix is symmetric, so by Eq.~\ref{eq:decomposition} the antisymmetric channel is unchanged, while reversal $\bB\mapsto \bB^{\top}$ negates it, and zero is the threshold fixed by the functional itself. Under an argument swap of a one-edge caption, which negates the edge, the pair decision $D(C^+)-D(C^-)$ has the sign of $D(C^+)$. Under a uniformly random permutation of the patch positions every centroid has the same expectation, the mean patch center, so $D$ has mean zero for every coupling.

\section{Reliability, Localization and Computational Cost}
\label{app:diagnostics}

\subsection{Risk and Coverage}
\label{app:measurement}
AUROC is the probability that a correct item carries a larger absolute margin than a wrong one, with exact score ties left out and confidence ties counted as half.

Tab.~\ref{tab:riskcov} reports the four binary direction-balanced sets and their pooled results for Fig.~\ref{fig:measurement}. For two-caption items, confidence is the absolute difference between the candidate scores. For VSR, it is $|D|$ for ADR and the absolute distance from the oracle threshold for the deployed score and likelihood readout. Pooling ranks these values without rescaling across datasets; it does not calibrate them as probabilities of correctness. The per-caption COCO scores come from a recompute of the published loop in the isolation directory whose decisions match the published file on every covered item in both passes.

\begin{table}[t]
\caption{Risk and coverage on the four binary direction-balanced sets. The pooled results use only the four direction-balanced sets ($5877$ items). The dual-encoder readouts use CLIP ViT-B/16, and the likelihood readout uses Qwen3.6-27B. (a) AUROC of the absolute margin as a ranking of correct above wrong items, with the same statistic under the derangement null in parentheses. (b) Accuracy of the most confident $30\%$ of the items. Pooling ranks the items of the four direction-balanced sets by one margin, so it is the accuracy a single threshold would obtain. The four-way What'sUp-A is excluded from this binary-set analysis.}
\label{tab:riskcov}
\centering\fontsize{7}{9}\selectfont\setlength{\tabcolsep}{0.75pt}
\begin{tabular*}{\columnwidth}{@{\extracolsep{\fill}}l@{\hspace{5pt}}ccccc@{}}
\toprule
readout & VG-Rel L/R & COCO & WU-B & VSR & pooled \\
$n$ & $3918$ & $404$ & $354$ & $1201$ & $5877$ \\
\midrule
\multicolumn{6}{@{}l}{\emph{(a) AUROC (derangement null)}} \\
\dirpix{} & 0.891 (0.510) & 0.904 (0.516) & 0.888 (0.449) & 0.713 (0.484) & 0.863 (0.504) \\
\dirdense{} & 0.880 (0.509) & 0.899 (0.483) & 0.970 (0.506) & 0.694 (0.478) & 0.844 (0.504) \\
deployed cosine & 0.509 (0.505) & 0.461 (0.461) & 0.512 (0.481) & 0.504 (0.487) & 0.514 (0.504) \\
MLLM likelihood, Qwen3.6-27B & 0.717 (0.496) & 0.726 (0.555) & 0.829 (0.508) & 0.568 (0.461) & 0.661 (0.498) \\
\midrule
\multicolumn{6}{@{}l}{\emph{(b) accuracy at $30\%$ coverage}} \\
\dirpix{} & 1.000 & 1.000 & 0.991 & 0.914 & 0.998 \\
\dirdense{} & 0.999 & 1.000 & 1.000 & 0.903 & 0.997 \\
deployed cosine & 0.495 & 0.446 & 0.575 & 0.547 & 0.516 \\
MLLM likelihood, Qwen3.6-27B & 0.916 & 0.967 & 0.972 & 0.689 & 0.849 \\
\bottomrule
\end{tabular*}
\end{table}

\paragraph{Example selection.}
The correct examples in Fig.~\ref{fig:measurement}(a) were selected for having the largest ADR margins among eligible traces. Eligibility requires ADR to be correct and the global score wrong, with both arguments parsed as scene nouns. For COCO examples, both centroids must fall inside their annotated object boxes; depth examples require the signed statistic \texttt{vdy\_pix} to exceed $0.5$.

\subsection{What'sUp-A Read by Two Encoders}
\label{app:wualoc}
Tab.~\ref{tab:wualoc} compares localization failures of \dirpix{} on CLIP ViT-B/16 and SigLIP2-So400m using the full What'sUp-A set. This noun-level diagnostic includes items excluded from the covered subset used in Tab.~\ref{tab:main}; its parsing checks and failure counts should not be read as a decomposition of the main-table errors. A failure is identified when the displacement between the localized centroids contradicts the annotated relation. SigLIP2 has $36$ such failures, compared with $147$ for CLIP, with fewer failures across most listed nouns. For cup, plate, cap and wineglass, SigLIP2 has no identified failures while CLIP does. This reduction is consistent with the stronger ADR performance of SigLIP2 on What'sUp-A in Tab.~\ref{tab:main}.

\begin{table}[t]
\caption{Relative-position failures grouped by subject noun on What'sUp-A, with both encoders evaluated on the same items. Rows show the thirteen nouns with at least twelve items, ordered by the ViT-B/16 failure count; the remaining forty-two are pooled. Counts indicate failures of the estimated relative positions, without identifying which object is mislocalized. SigLIP2 has no such failures for 29 of the fifty-five subject-noun groups.}
\label{tab:wualoc}
\centering\small
\begin{tabular*}{\columnwidth}{@{\extracolsep{\fill}}lccc@{}}
\toprule
subject noun & items & ViT-B/16 & SigLIP2 \\
\midrule
bottle & 28 & 8 & 2 \\
bowl & 12 & 6 & 3 \\
cup & 12 & 6 & 0 \\
orange & 12 & 5 & 2 \\
pillow & 12 & 5 & 2 \\
plate & 12 & 5 & 0 \\
sunglasses & 12 & 5 & 1 \\
cap & 12 & 4 & 0 \\
mug & 12 & 4 & 1 \\
yarn & 12 & 4 & 1 \\
book & 12 & 3 & 1 \\
scarf & 12 & 2 & 2 \\
wineglass & 12 & 2 & 0 \\
\midrule
the other forty-two nouns & 240 & 88 & 21 \\
all & 412 & 147 & 36 \\
\bottomrule
\end{tabular*}
\end{table}

\subsection{VSR and Its Reference Frames}
\label{app:vsrmain}
Tab.~\ref{tab:vsr} reports results on the $1201$ covered horizontal items from VSR~\citep{liu2023visual}. With CLIP ViT-B/16, \dirpix{} achieves an accuracy of $0.7386$ and a grounded gain of $+26.1$ percentage points. The global score reaches $0.5387$ even with an oracle threshold.

The reference-frame annotations in VSR help explain the remaining errors. On true statements, ADR agrees with relations defined from the viewer's perspective at a rate of $0.9290$, but with those defined from the object's perspective at only $0.2239$. ADR measures displacement in image coordinates and therefore follows the viewer's perspective. This mismatch with object-centered annotations explains some of its errors on VSR.

\begin{table}[h]
\caption{VSR covered horizontal subset ($n{=}1201$, whole images, third-party labels) and its decomposition by the benchmark's reference-frame annotation. Frame rows are true-item consistency rates, the false-statement row is the shared specificity, and balanced accuracy pairs each frame's consistency with that specificity.}
\label{tab:vsr}
\centering\small
\begin{tabular}{@{}lccc@{}}
\toprule
 & CLIP B/16 & null & NegCLIP \\
\midrule
\dirpix{}, accuracy & 0.7386 \ci{0.714}{0.763} & 0.4779 & 0.7394 \\
global, oracle threshold & 0.5387 & 0.5321 & 0.5287 \\
\midrule
viewer-frame consistency ($n{=}507$) & 0.9290 \ci{0.905}{0.951} & 0.5069 & 0.9172 \\
object-frame consistency ($n{=}67$) & 0.2239 \ci{0.134}{0.328} & 0.4776 & 0.2239 \\
both frames ($n{=}30$) & 0.8667 & 0.3333 & 0.9667 \\
specificity on false statements ($n{=}597$) & 0.6281 & 0.4606 & 0.6348 \\
balanced accuracy, viewer / object & 0.7786 / 0.4260 & 0.4838 / 0.4691 & 0.7760 / 0.4294 \\
\bottomrule
\end{tabular}
\end{table}

\subsection{Temperature Sweep}
\label{app:epssweep}
Tab.~\ref{tab:eps} compares both ADR variants at four temperatures from $0.02$ to $0.2$ using the same frozen CLIP ViT-B/16 features. Our default is $\varepsilon=0.05$. Across this range, accuracy on VG-Relation left/right varies by $0.1$ percentage points for \dirpix{} and $1.6$ for \dirdense{}. On What'sUp-A, the corresponding ranges are $2.9$ and $28.7$ points. For \dirdense{}, accuracy on What'sUp-A falls from $0.6078$ at $\varepsilon=0.02$ to $0.3211$ at $0.2$. These results show that \dirpix{} is less sensitive to temperature, although \dirdense{} benefits more from concentrated patch weights on What'sUp-A.

\begin{table}[t]
\caption{Temperature sensitivity on frozen CLIP ViT-B/16 features. Each variant reports acc., n.acc., and gain as defined in Tab.~\ref{tab:main}. Accuracies are in percent and gains in percentage points. The default temperature is $\varepsilon=0.05$.}
\label{tab:eps}
\centering\fontsize{8}{9.5}\selectfont\setlength{\tabcolsep}{1pt}
\begin{tabular*}{\textwidth}{@{\extracolsep{\fill}}l*{12}{c}@{}}
\toprule
 & \multicolumn{6}{c}{VG-Rel L/R, $n{=}3918$} & \multicolumn{6}{c}{What'sUp-A, $n{=}408$} \\
\cmidrule(lr){2-7}\cmidrule(lr){8-13}
\multirow{2}{*}{$\varepsilon$} & \multicolumn{3}{c}{\dirpix{}} & \multicolumn{3}{c}{\dirdense{}} & \multicolumn{3}{c}{\dirpix{}} & \multicolumn{3}{c}{\dirdense{}} \\
\cmidrule(lr){2-4}\cmidrule(lr){5-7}\cmidrule(lr){8-10}\cmidrule(lr){11-13}
 & acc. & n.acc. & gain & acc. & n.acc. & gain & acc. & n.acc. & gain & acc. & n.acc. & gain \\
\midrule
$0.02$ & 90.35 & 49.26 & $+41.1$ & 88.72 & 50.66 & $+38.1$ & 29.41 & 25.25 & $+4.2$ & 60.78 & 25.00 & $+35.8$ \\
$0.05$ & 90.45 & 49.03 & $+41.4$ & 90.22 & 49.44 & $+40.8$ & 31.62 & 25.74 & $+5.9$ & 41.18 & 25.25 & $+15.9$ \\
$0.1$ & 90.40 & 48.90 & $+41.5$ & 90.25 & 48.80 & $+41.5$ & 31.86 & 24.75 & $+7.1$ & 33.82 & 25.00 & $+8.8$ \\
$0.2$ & 90.40 & 48.88 & $+41.5$ & 90.30 & 48.26 & $+42.0$ & 32.35 & 24.75 & $+7.6$ & 32.11 & 25.00 & $+7.1$ \\
\bottomrule
\end{tabular*}
\end{table}

\subsection{Computational Cost}
\label{app:cost}
Tab.~\ref{tab:cost} compares the computational cost of ADR and the evaluated readouts.

\begin{table}[t]
\caption{Computational cost per item, averaged over the same 50 items on one NVIDIA H200 after one warm-up call. All ratios are relative to the deployed CLIP.}
\label{tab:cost}
\centering\fontsize{7}{9}\selectfont\setlength{\tabcolsep}{2.5pt}
\begin{tabular*}{\columnwidth}{@{\extracolsep{\fill}}lcccccccc@{}}
\toprule
\multirow{2}{*}{Readout} & \multirow{2}{*}{param.} & param.\ vs.\ & \multirow{2}{*}{FLOPs} & FLOPs vs.\ & \multirow{2}{*}{mem.\ (GiB)} & mem.\ vs.\ & \multirow{2}{*}{time (ms)} & time vs.\ \\
 & & CLIP & & CLIP & & CLIP & & CLIP \\
\midrule
deployed / CLIP & 0.15 B & 1.0$\times$ & $3.45{\times}10^{10}$ & 1.0$\times$ & 0.60 & 1.0$\times$ & 9.28 & 1.0$\times$ \\
ADR / CLIP & 0.15 B & 1.0$\times$ & $1.15{\times}10^{11}$ & 3.3$\times$ & 0.66 & 1.1$\times$ & 20.78 & 2.2$\times$ \\
half-crop / CLIP & 0.15 B & 1.0$\times$ & $1.38{\times}10^{11}$ & 4.0$\times$ & 0.60 & 1.0$\times$ & 43.83 & 4.7$\times$ \\
\midrule
deployed / SigLIP2 & 1.14 B & 7.6$\times$ & $7.39{\times}10^{11}$ & 21.4$\times$ & 4.35 & 7.3$\times$ & 35.97 & 3.9$\times$ \\
ADR / SigLIP2 & 1.14 B & 7.6$\times$ & $1.05{\times}10^{12}$ & 30.4$\times$ & 4.40 & 7.3$\times$ & 52.98 & 5.7$\times$ \\
half-crop / SigLIP2 & 1.14 B & 7.6$\times$ & $2.96{\times}10^{12}$ & 85.8$\times$ & 4.35 & 7.3$\times$ & 140.44 & 15.1$\times$ \\
\midrule
BLIP-ITM base & 0.22 B & 1.5$\times$ & $2.60{\times}10^{11}$ & 7.5$\times$ & 0.96 & 1.6$\times$ & 20.29 & 2.2$\times$ \\
LLaVA-1.5-7B likelihood & 7.06 B & 47.1$\times$ & $1.73{\times}10^{13}$ & 501.4$\times$ & 13.55 & 22.6$\times$ & 67.83 & 7.3$\times$ \\
LLaVA-1.5-7B chat & 7.06 B & 47.1$\times$ & $1.78{\times}10^{13}$ & 515.9$\times$ & 13.57 & 22.6$\times$ & 90.69 & 9.8$\times$ \\
Qwen3.6-27B likelihood & 27.36 B & 182.4$\times$ & $1.71{\times}10^{13}$ & 495.7$\times$ & 51.30 & 85.5$\times$ & 505.54 & 54.5$\times$ \\
Qwen3.6-27B chat & 27.36 B & 182.4$\times$ & $1.90{\times}10^{13}$ & 550.7$\times$ & 51.25 & 85.4$\times$ & 611.11 & 65.9$\times$ \\
\bottomrule
\end{tabular*}
\end{table}

\section{ADR Implementation}
\label{app:implementation}

\algrenewcommand{\algorithmiccomment}[1]{\hfill{\color{gray}\footnotesize\texttt{//} #1}}
\begin{algorithm}[t]
\caption{Antisymmetric Displacement Readout (ADR)}
\label{alg:channel}
\begin{algorithmic}[1]
  \Statex \textbf{Input} similarity matrix $\bS$, patch centers $\{\bp_i\}$, temperature $\varepsilon$, relations $\{(s_r,o_r,\be_r)\}$
  \Statex \textbf{Output} score $D$, alignment weights $\tilde{\pi}$ and centroids $\{\bc_j\}$
  \State $\pi_0\gets\mathrm{Sinkhorn}(\exp(\bS/\varepsilon))$
  \Comment{Eq.~\ref{eq:ot}}
  \State $\tilde{\pi}\gets[\pi_0-{1} / {(Nm)}]_+$
  \Comment{Retain excess mass}
  \State $\bc_j\gets
  {(\sum_i\tilde{\pi}_{ij}\bp_i)} /
  {\max(\sum_i\tilde{\pi}_{ij},10^{-30})}$
  \quad for each noun $j$
  \Comment{Calculate centroids}
  \State $D\gets\sum_r\be_r^\top(\bc_{s_r}-\bc_{o_r})$
  \Comment{Prop.~\ref{prop:centroid}}
\end{algorithmic}
\end{algorithm}

\subsection{Patch Feature Extraction}
\label{app:poolingrule}
We extract patch features using each encoder's existing weights. For CLS models such as CLIP, we apply the final attention block's input normalization, value projection and output projection to each token. This bypasses attention mixing, residual connections and the final block's feed-forward network. The patch outputs then pass through the encoder's final normalization and image projection.

For MAP models such as SigLIP, we keep all Transformer blocks and apply the pooling head separately to each final patch token. With only one token, the attention weight is $1$. For a token $\mathbf{y}_i$ after trunk normalization, the head gives $h(\mathbf{y}_i)=\mathrm{fc\_norm}\!\left(\mathrm{proj}(v(\mathbf{y}_i))+\mathrm{mlp}(\mathrm{norm}(\mathrm{proj}(v(\mathbf{y}_i))))\right)$. We retain the pooling head's MLP and residual connection.

Both paths produce features that we normalize to unit length to obtain $\bx_i$ for the similarity matrix $\bS$. As an implementation check, we compare the reconstructed full-token forward path with the model's original pooled embedding.

\subsection{Column Weights}
\label{app:localizers}
Prop.~\ref{prop:centroid} allows ADR to use any nonnegative alignment weights with positive column masses. Tab.~\ref{tab:weights} lists the two choices used in the main comparisons. They differ in how they localize words but use the same signed displacement score. In the OT coupling $\pi_0$, each patch distributes a fixed mass of $1/N$ across the words, so words compete for a shared patch budget. The default weights $\tilde{\pi}$ retain only the mass above the uniform baseline and no longer preserve the fixed row and column masses. Dense weights instead localize each word independently by averaging patch positions.

\begin{table}[h]
\caption{The column weights of ADR. Every weight feeds the same relation axis, sign rule and zero threshold and differs only in how patch mass is assigned to words. Their accuracies are reported in Tab.~\ref{tab:allreadouts}.}
\label{tab:weights}
\centering\fontsize{7}{9}\selectfont\setlength{\tabcolsep}{4pt}
\begin{tabular*}{\columnwidth}{@{\extracolsep{\fill}}lp{\dimexpr\columnwidth-64pt\relax}@{}}
\toprule
\multicolumn{2}{@{}l}{\dirpix{}, the default} \\
definition & \raggedright the excess $\tilde{\pi}=[\pi_0-\frac{1}{Nm}]_+$ of the Sinkhorn coupling\tabularnewline
what it buys & \raggedright The underlying coupling $\pi_0$ constrains both marginals, so words share a patch budget. Subtracting the uniform baseline removes its contribution to the centroid. The retained weights $\tilde{\pi}$ no longer preserve either marginal. Their column masses can be inspected alongside the localization maps, but do not by themselves establish localization quality.\tabularnewline
what it costs & \raggedright Sinkhorn scaling requires iteration. The shared patch budget in $\pi_0$ couples word localizations and can limit how well objects of different sizes are represented. A noun can lose all its excess mass and then has no centroid (App.~\ref{app:excess}).\tabularnewline
\midrule
\multicolumn{2}{@{}l}{\dirdense{}} \\
definition & \raggedright the softmax of $\bS/\varepsilon$ over patches, computed independently for each word\tabularnewline
what it buys & \raggedright Closed form with no iteration. Each word is localized on its own, so a large object cannot drain mass from a small one. The column mass is fixed, so every centroid is defined.\tabularnewline
what it costs & \raggedright There is no shared patch budget, so several words can concentrate their weights on the same patch. Each column has fixed total mass, so this total does not distinguish word localizations. The weights still define patch-word alignments for the centroid score.\tabularnewline
\bottomrule
\end{tabular*}
\end{table}

\subsection{Excess Coupling and the Degenerate Column}
\label{app:excess}
The default variant computes centroids from the excess coupling $\tilde{\pi}=[\pi_0-\tfrac{1}{Nm}]_+$. This is a fixed entrywise function of $\pi_0$, so nothing is re-optimized and Prop.~\ref{prop:centroid} applies wherever its column masses are positive. A column of $\tilde{\pi}$ is empty only if its column of $\pi_0$ is exactly uniform. The implementation floors the column mass at $10^{-30}$ and counts these cases. None were observed in published runs.

\subsection{Sinkhorn Settings}
\label{app:sinkhorn}
The OT coupling is computed by Sinkhorn scaling~\citep{sinkhorn1967concerning} of $\exp((\bS-\max \bS)/\varepsilon)$ with uniform marginals. Scaling runs at most $200$ iterations and stops when both scaling vectors change by less than $10^{-9}$ in the max norm. One iteration uses two matrix-vector products on the $N\times m$ table, about $4Nm\approx1.2\times10^{4}$ flops for $N=588$, $m=5$. At this illustrative patch count, the cap costs at most $2.4\times10^{6}$ flops, below $10^{-4}$ of the several $10^{10}$ flops of a ViT-B/16 forward pass. Patch counts vary with the input aspect ratio and resolution.

\subsection{Depth Relations}
\label{app:depth}
For front/behind relations, ADR uses the image vertical axis: in front of is $+y$ (lower in the image), and behind is $-y$ (higher in the image). This axis was chosen after exploratory VG-Relation analysis and fixed before evaluation on disjoint VG-Relation images and What'sUp-B. It uses image position as a depth cue rather than estimating three-dimensional depth. The positive-displacement interpretation in Sec.~\ref{sec:method} therefore applies to this proxy for front/behind: correct object localization alone does not guarantee the correct three-dimensional relation. The absence of training refers to encoder weights and learned readout parameters; this axis was selected manually from the exploratory analysis.

\section{Evaluation and Reproducibility}
\label{app:protocol}

\subsection{Data Processing and Released Assets}
\label{app:data}
Benchmark preparation has four steps. Tab.~\ref{tab:data-overview} summarizes the subsets used in the main-text results, and Tab.~\ref{tab:data} gives sampling and coverage details for the spatial sets. Sampling. The main VG-Relation comparisons use a fixed uniform sample of 6000 items. No other benchmark is subsampled. Of ARO we use VG-Relation only, because its attribution and order subsets vary attribute binding and word order rather than a directional relation. Subsets. The primary subset is $\mathrm{rel}\in\{\text{left},\text{right}\}$ on VG-Relation, selected by the benchmark's relation field with no coverage filter, so each item carries a caption and its swap. VSR supplies its own subset labels and reference-frame annotations. The prior analysis also includes $408$ VG-Relation verb items and $949$ VALSE actant-swap pairs. The VALSE pairs exchange the agent and patient of a verb. Parsing. Every item receives a derangement partner from the same benchmark by the cyclic shift of App.~\ref{app:nulls}, $K{=}997$ on VG-Relation and $K{=}101$ on the $k$-way cells. We release the item lists with their subset labels, the derangement partner indices, the relation phrases and parser, the scoring and oracle-threshold scripts, the per-item audit records, and the preregistration files with their hash manifests. Images are loaded from the benchmarks' official sources and are not redistributed.

VG-Relation images are the relation crops of ARO's official loader, and every other set enters as the whole image. Dual-encoder readouts receive the original caption text; MLLM task instructions are specified in App.~\ref{app:mllmforms}.

\begin{table}[t]
\caption{Evaluation subsets from ARO~\citep{yuksekgonul2023when}, What'sUp~\citep{kamath2023whatsup}, VSR~\citep{liu2023visual} and VALSE~\citep{parcalabescu2022valse}. Classes gives the number of possible answers, and Items counts evaluated examples. Data processing details are given in App.~\ref{app:data}.}
\label{tab:sets}
\label{tab:data-overview}
\centering
\fontsize{8}{9.5}
\selectfont
\begin{tabular*}{\columnwidth}{@{\extracolsep{\fill}}lllclc@{}}
\toprule
Subset & Dataset & Relations & Classes & Question type & Items \\
\midrule
\multicolumn{6}{@{}l}{\emph{(a) Direction-balanced subsets (Tabs.~\ref{tab:main} and~\ref{tab:alternatives})}} \\
VG-Relation left/right & ARO & Horizontal & 2 & Multiple choice & 3918 \\
COCO-spatial & What'sUp & Horizontal, vertical & 2 & Multiple choice & 404 \\
What'sUp-A & What'sUp & Horizontal, vertical & 4 & Multiple choice & 408 \\
What'sUp-B & What'sUp & Horizontal, depth & 2 & Multiple choice & 354 \\
VSR horizontal & VSR & Horizontal & 2 & True/false & 1201 \\
\midrule
\multicolumn{6}{@{}l}{\emph{(b) Prior-biased subsets (Tab.~\ref{tab:prior-all})}} \\
VG-Relation vertical & ARO & Vertical & 2 & Multiple choice & 800 \\
VG-Relation verbs & ARO & Verb relations & 2 & Multiple choice & 408 \\
VALSE actant swap & VALSE & Verb argument roles & 2 & Multiple choice & 949 \\
\bottomrule
\end{tabular*}
\end{table}

\begin{table}[t]
\caption{Evaluation sets. Each set is a deterministic selection from a public benchmark. The VG-Relation results use the 6000-item sample. Sources are ARO VG-Relation \citep{yuksekgonul2023when}, the What'sUp release \citep{kamath2023whatsup}, which also supplies COCO-spatial, and VSR \citep{liu2023visual}. Covered counts the items whose caption the fixed parsing rule covers, and VSR carries the benchmark's own reference-frame labels. COCO-spatial's $2247$ single-object captions lie outside the two-argument readout.}
\label{tab:data}
\centering\fontsize{7}{9}\selectfont\setlength{\tabcolsep}{3pt}
\begin{tabular*}{\columnwidth}{@{\extracolsep{\fill}}llll@{}}
\toprule
evaluation set & source & selection & items \\
\midrule
VG-Relation left/right, 6000-item cell & ARO & relation field, 6000-item pick & 3918 \\
VG-Relation vertical, 6000-item cell & ARO & relation field, 6000-item pick & 800 \\
COCO-spatial two-object & What'sUp & all two-object pairs & 440 (404 covered) \\
What'sUp-A & What'sUp & full set, 4-way & 412 (408 covered) \\
What'sUp-B & What'sUp & front/behind and left/right pairs & 354 \\
VSR horizontal & VSR & horizontal subset & 1224 (1201 covered) \\
\bottomrule
\end{tabular*}
\end{table}

\subsection{Models}
\label{app:models}
Every model here is frozen. The dual encoders of block (a) are read twice on the same features, by the deployed
pooled score the benchmark reports and by ADR, as is the ITC branch of BLIP base. The matching heads of block (b)
and the MLLMs of block (c) lie outside the theorems of Sec.~\ref{sec:impossibility} and are scored in their own
native forms.

\begin{table}[t]
\caption{Every model evaluated in this paper. $\star$ marks the nine score families of
Table~\ref{tab:main}, where BLIP enters through the ITC branch of its base checkpoint. Loader
settings are in the released code.}
\label{tab:models}
\centering
\small
\begin{tabular}{@{}ll@{}}
\toprule
model & note \\
\midrule
\multicolumn{2}{@{}l}{(a) frozen dual encoders, deployed pooled score and ADR (Tabs.~\ref{tab:main} and~\ref{tab:allreadouts})} \\
CLIP ViT-B/32 & \\
CLIP ViT-B/16$^{\star}$ & \\
CLIP ViT-L/14 & \\
LAION-2B ViT-B/16$^{\star}$ & App.~\ref{app:grid} \\
DFN5B ViT-H/14$^{\star}$ & \\
MetaCLIP ViT-H/14$^{\star}$ & \\
SigLIP-B/16 & pooling-native \\
SigLIP-So400m & pooling-native \\
SigLIP2-So400m$^{\star}$ & pooling-native \\
NegCLIP$^{\star}$ & fine-tuned from ViT-B/32 \\
CLoVe$^{\star}$ & fine-tuned from ViT-B/32, App.~\ref{app:grid} \\
CE-CLIP$^{\star}$ & fine-tuned from ViT-B/32, App.~\ref{app:grid} \\
\midrule
\multicolumn{2}{@{}l}{(b) cross-encoder matching heads, read by ITM (Table~\ref{tab:alternatives})} \\
BLIP-ITM base$^{\star}$ & \\
BLIP-ITM large & \\
BLIP-2 ViT-g & \\
FLAVA-full & \\
\midrule
\multicolumn{2}{@{}l}{(c) MLLMs (Tabs.~\ref{tab:alternatives} and~\ref{tab:allreadouts}; settings in App.~\ref{app:mllmforms})} \\
LLaVA-1.5-7B & \citep{liu2024improved} \\
Qwen3.6-27B & \citep{qwen36_27b} \\
\bottomrule
\end{tabular}
\end{table}

\paragraph{Matching-head settings.}
We measured BLIP-ITM base and large (COCO fine-tuned), BLIP-2 ITM ViT-g and FLAVA-full with its pretrain-only ITM head \citep{li2022blip,li2023blip,singh2022flava}. They lie outside the theorems because joint attention can read word order. ITM is the matched-class log-softmax of the two-way match head, scored on the paper's items, forced choices and derangement nulls, with $K{=}997$ on VG-Relation and $K{=}101$ on the $k$-way cells, in fp32. ITM requires one joint forward pass per image--text pair and cannot be precomputed.

\subsection{Scoring}
\label{app:scoring}
Ties are margins within $10^{-9}$ and score $\tfrac12$, split evenly in $k$-way items. A pair is decided by the margin $D(C^+)-D(C^-)$, a $k$-way item by the largest $D$ with uncovered options at $D{=}0$, and a VSR true/false item by the sign rule, True iff $D>10^{-9}$. Baseline true/false readouts use a monotone threshold chosen on the evaluation labels, with both constant classifiers included, so their accuracy is at least $0.5$ and overstated in their favor. For these VSR baselines, the threshold is optimized separately under the original pairing and the derangement null. The unchanged decision rule in Sec.~\ref{sec:deflation} is this threshold-selection procedure, not a shared numerical threshold. Their grounded gain is the difference between the two separately optimized accuracies.

\paragraph{Statistics.}
\label{app:statistics}
Confidence intervals are 10k-resample item-level bootstraps within each subset (seed 0), paired over the shared items for paired differences.

\subsection{MLLM Question Forms}
\label{app:mllmforms}
The MLLM readouts are run plainly, with no system prompt, no worked example, no chain of thought, no retry and no self-consistency, greedy decoding with a four-token budget, and the optional thinking mode of the one model that has it turned off. In the likelihood form the readout is the caption log-probability under the fixed prefix \texttt{Write a one-sentence description of the image.}, summed over the caption's own tokens. Among $k$ candidates the largest log-probability wins under the same $10^{-9}$ tie rule as every other readout, and on VSR it is read through a single threshold chosen on the evaluation labels, the concession the deployed score receives there. In the chat form two candidates are asked with the question published for the VG-Relation cell, \texttt{\small Which of the following captions correctly describes the image? A: \{a\} B: \{b\} Answer with a single letter, A or B.}, a $k$-way item with $k$ lettered lines, and a true or false item with \texttt{Is the following statement true of the image? \{statement\} Answer with a single word, yes or no.} Every item is asked in both candidate orders and the two answers are averaged. The reply is read for the first standalone option letter or yes or no, and a reply with none scores $1/k$. The per-item replies of both orders ship with the release.

\subsection{Half-Crop Intervention on the Five Balanced Sets}
\label{app:halfcrop}
The image is cut into its left, right, top and bottom halves, each half is encoded once, and a caption is scored by comparing its subject and object words with the two halves on either side of its relation axis. On VG-Relation left/right the margin is the true caption's score minus the false caption's. On COCO-spatial and What'sUp-A every candidate is scored along its own axis and the highest score wins with ties split, on the $404$ and $408$ items where every caption parses. What'sUp-B uses the same $354$ items as ADR. Left/right relations use horizontal image halves, while front/behind relations follow the depth rule, with in front of as lower in the image and behind as higher (App.~\ref{app:depth}). On VSR the margin is the statement's score minus its score with subject and object exchanged, read by its sign on the covered $1201$ items. Every set is scored against its own derangement null (App.~\ref{app:nulls}).

\subsection{Nulls}
\label{app:nulls}
The derangement is a deterministic cyclic shift over the item list with same-source exclusion. No item receives its original source image or a duplicate of it.

\end{document}